\documentclass{article}
\usepackage[utf8]{inputenc}

\usepackage{amsmath}
\usepackage{amssymb, mathrsfs}
\usepackage{amsthm,color}

\newcommand{\R}{\mathbb{R}}
\newcommand{\Rd}{\mathbb{R}^d}
\newcommand{\M}{\mathcal{M}}
\newcommand{\KL}[2]{\mathrm{KL}(#1|#2)}

\newcommand{\lmu}{\underline{\mu}}
\newcommand{\umu}{\overline{\mu}}

\newcommand{\cost}{\mathcal{C}}
\newcommand{\Y}{\mathcal{Y}}

\newcommand{\Proj}[1]{\mathrm{Proj}_{#1}}
\newcommand{\spt}[1]{\mathrm{spt}(#1)}
\newcommand{\diam}[1]{\mathrm{diam}(#1)}

\newtheorem{Thm}{Theorem}[section]
\newtheorem{Lem}[Thm]{Lemma}
\newtheorem{Prop}[Thm]{Proposition}
\newtheorem{Cor}[Thm]{Corollary}

\theoremstyle{definition}
\newtheorem{Def}[Thm]{Definition}

\theoremstyle{remark}
\newtheorem{Rmk}[Thm]{Remark}

\newcommand\numberthis{\addtocounter{equation}{1}\tag{\theequation}}

\title{Convergence analysis of t-SNE as a gradient flow for point cloud on a manifold}
\author{Seonghyeon Jeong$^1$, Hau-Tieng Wu$^2$}
\date{\small $^1$ National Center for Theoretical Sciences, National Taiwan University, Taipei, Taiwan\\%
    $^2$ Courant Institute of Mathematical Sciences, New York University, New York, NY, 10012 USA}

\begin{document}

\maketitle

\begin{abstract}
We present a theoretical foundation regarding the boundedness of the t-SNE algorithm. t-SNE employs gradient descent iteration with Kullback-Leibler (KL) divergence as the objective function, aiming to identify a set of points that closely resemble the original data points in a high-dimensional space, minimizing KL divergence. Investigating t-SNE properties such as perplexity and affinity under a weak convergence assumption on the sampled dataset, we examine the behavior of points generated by t-SNE under continuous gradient flow. Demonstrating that points generated by t-SNE remain bounded, we leverage this insight to establish the existence of a minimizer for KL divergence.
\end{abstract}

\section{Introduction}
In data analysis, dimension reduction aims to represent high-dimensional points in a Euclidean space with significantly lower dimensions (e.g., 2 or 3). Various algorithms exist for this purpose, including ISOMAP \cite{tenenbaum2000global}, Locally Linear Embedding (LLE) \cite{roweis2000nonlinear}, Eigenmap \cite{belkin2003laplacian}, Diffusion Map (DM) \cite{Coifman2006}, and its variations like bi-stochastic kernel \cite{marshall2019manifold} and ROSELAND \cite{shen2022robust}, Vector Diffusion Map (VDM) \cite{singer2012vector}, Stochastic Neighborhood Embedding (SNE) \cite{Hinton_Roweis_2003}, t-distributed SNE (t-SNE) \cite{JMLR:v9:vandermaaten08a}, LargeVis \cite{LargeVis}, and UMAP \cite{UMAP}, among others. These algorithms can be broadly categorized into two groups depending on their underlying methodologies and problem-solving paradigms: spectral methods (e.g., ISOMAP, LLE, Hessian LLE, Eigenmap, DM, and VDM) and iteration-based methods (e.g., SNE, t-SNE, LargeVis, and UMAP). Among spectral methods, LLE, Eigenmap and DM are mainly based on graph Laplacian, and VDM is a generalization of DM by taking the connection structure into account, while ISOMAP explores the data structure from the other angle. Among iteration-based methods, in addition to t-SNE, many variations of SNE are available, including using kernels with heavier tails to quantify the embedded points \cite{yang2009heavy,kobak2019heavy}, or using the f-divergence to capture the intrinsic structure for the embedding \cite{im2018stochastic}, among many others. If we view SNE as a attraction-repulsion force-based approach \cite{bohm2022attraction}, the more recently introduced LargeVis and UMAP could be considered as variations of SNE as well, where the repulsive forces are modified for a sampling-based stochastic optimisation. This article primarily focuses on studying t-SNE.

Among numerous dimension reduction algorithms, t-SNE stands out for its practical performance \cite{amir2013visne,wattenberg2016use,kobak2019art,linderman2019fast,belkina2019automated}. However, unlike the rich theoretical support for spectral method-based algorithms, t-SNE has limited theoretical backing, possibly due to its more challenging iteration-based nature, and this might mislead scientific developments \cite{chari2023specious}. Considering t-SNE's popularity, it is an urgent need to establish its theoretical supports. Here we summarize existing theoretical studies on t-SNE. Researchers have explored how original t-SNE \cite{Shaham2017StochasticNE} and t-SNE with early exacerbation \cite{arora2018analysis,doi:10.1137/18M1216134} perform in the clustering mission, showing that highly clustered data results in a clustered output. Under appropriate parameter choices (asymptotically small learning rate and large early exacerbation with a constant product), it has been shown that t-SNE with early exacerbation behaves akin to a spectral clustering algorithm \cite{doi:10.1137/18M1216134,cai2022theoretical}. In addition to a quantitative explanation about the need to terminate early exacerbation immaturely, the authors in \cite{cai2022theoretical} dig into the dynamics of ordinary t-SNE and study the amplification and stabilization phases under some conditions.
t-SNE can also be analyzed as a force-based method utilizing repulsive and attractive forces between data points  \cite{forcefulcolor,maytheforce}, and the associated ``force vector'' has been explored as an important additional feature of t-SNE and other similar algorithms. In \cite{forcefulcolor}, the mean field limits of t-SNE has been studied under the $k$ regular random graph model.

In this paper, we study fundamental aspects of t-SNE, focusing on its behavior as a continous gradient flow when the high-dimensional input dataset is sampled from a manifold. We pose a fundamental question: \textit{As an iterative algorithm, does t-SNE yield any data point diverging to $\infty$ during iteration?} To the best of our knowledge, this remains an open question necessitating exploration. Given that the Kullback-Leibler (KL)-divergence serves as the cost function in t-SNE, addressing this question is crucial due to the non-convex nature of the optimization problem involved. Under mild conditions, we investigate key properties of perplexity, a critical parameter in t-SNE, and demonstrate that none of the embedded points generated by t-SNE diverges to $\infty$. This is the first main theorem of this paper, which is sketched as:
\begin{Thm}[Main theorem 1, rough statement]
The points in $\R^2$ generated by t-SNE are uniformly bounded. 
\end{Thm}
With the boundedness result, we also prove the existence of a global minimizer of the cost function used in the t-SNE iteration.
\begin{Thm}[Main theorem 2, rough statement]
There exists a global minimizer of the KL-divergence.
\end{Thm}
Therefore, even though the KL-divergence does not have convexity, we can still show that there is a set of points that minimize the KL-divergence.

The proof's concept involves exploring the behavior of mutual distances among points $\{ y_i \}$ generated by t-SNE. This exploration hinges on various properties of the perplexity parameter, which inherently holds its own significance. Given that the center of mass of these points remains fixed (refer to \eqref{prop: center of mass preserve}), any unbounded behavior in $\{ y_i \}$ must be reflected in a mutual distance between points diverging to $\infty$. Employing the gradient descent equation, we compute the derivative of $\sum \| y_i - y_j \|^2$ and leverage the structures of affinities $p_{ij}$ and $q_{ij}$ to extract valuable insights. The gradient flow equation provides key information about mutual distances of points generated by the algorithm, which, when combined with other algorithmic structures, enables deductions about the embedded points' behavior. While the treatment is generic and applicable to other gradient flow-based algorithms, adaptation for different algorithms may require adjustments, given the computational dependencies on t-SNE's specific structure in this paper.

This paper is organized as follows. In section 2, we review the t-SNE. we present the factors that compose the algorithm, and explain the steps of the algorithm. In section 3, we present the setting which we will use for analyzing t-SNE, and prove some properties of the perplexity parameter. 
In section 4, we show our main theorems.
First, we assume that there is a point that does not stay bounded, and observe what happens on the mutual distance of the points generated by t-SNE. Then, we compute some quantities about the affinities $p_{ij}$ associated with the input high-dimensional dataset and $q_{ij}$ associated with the embedded points generated by t-SNE. Finally, we present the precise statement and the proof of the main theorems.

Through out the paper, $\mu$ will be a probability measure in high dimensional Euclidean space $\Rd$, where $d\in\mathbb{N}$ is in general large, and we will use $x_i$ to denote the given data points $\Rd$ which are sampled from the measure $\mu$. We will use $y_i$ to denote the points in $\R^2$. For $q\in \mathbb{N}$, the $q$-dim Euclidean ball with radius $r>0$ centered at $x$ is denoted as $B^q_r(x)$. We use the notation $a\sim b$ to indicate that two quantities $a,b> 0$ are of the same order when we do not need the precise implied constants.

\section{t-SNE algorithm}
In this section, we review the t-SNE algorithm, which is a non-linear dimension reduction method that may not retain all information from the original data points. The key step in t-SNE involves defining and quantifying the specific information to preserve within the algorithm. In this case, the focus is on preserving the similarity among data points, and we quantify this similarity with the following definition.
\begin{Def}
\label{Def: affinity p cond}
Let $\{ x_i \}_{i=1}^n$ be a set of $n$ points in $\Rd$. We define the conditional affinity $p_{j|i}$ by
\begin{equation}\label{eqn: def p cond}
p_{j|i} =\frac{\exp(-\| x_i - x_j\|^2 / 2\sigma_i^2 )}{\sum_{k \neq i} \exp(-\| x_i - x_k \|^2 / 2\sigma_i^2)}\,,
\end{equation}
where $\sigma_i>0$, $i=1,\ldots,n$ and $j\neq i$.
\end{Def}
Note that the affinity $p_{j|i}$ defines a probability $P_i$ on $\{x_i\}_{i\neq j}$ that depends on $\sigma_i$; that is, 
\begin{equation*}
    P_i(\{ j \}) = p_{j|i}
\end{equation*}
when $j\neq i$.
The constant $\sigma_i$ is determined by a given parameter called \emph{perplexity}, denoted as $Perp$, by the following equation:
\begin{equation}
\label{eqn: perp role}
    Perp = 2^{H(P_i)},
\end{equation}
where $H(P_i)$ is the \emph{Shannon entropy} of $P_i$ measured in bits
\begin{equation}
\label{eqn: def Shannon}
    H(P_i) = - \sum_{j|j \neq i} p_{j|i} \log_2 p_{j|i}.
\end{equation}
By declaring a value of the perplexity $Perp$ in a certain range, below we will show that the equation \eqref{eqn: perp role} has a unique solution $\sigma_i$. Specifically, in Section \ref{section perp analysis} we will compute the range of $Perp$ and show the existence and uniqueness of $\sigma_i$ that satisfies \eqref{eqn: perp role}.

Clearly, in general $p_{j|i}$ is not symmetric. In \cite{JMLR:v9:vandermaaten08a}, the authors defined the \emph{symmetrized affinity $p_{ij}$} by 
\begin{equation}\label{eqn: def p}
    p_{ij} = \frac{1}{2n}(p_{j|i} + p_{i|j})\,,
\end{equation}
where $i\neq j$,
and used this for the affinity of data points in the high dimensional space $\Rd$. Note that we do not defined $p_{ii}$, and the symmetrized affinity $p_{ij}$ defines another probability on the set $\{(x_i,x_j)|i\neq j\}\subset \mathbb{R}^d\times \mathbb{R}^d$ as
\begin{equation*}
    P(\{ (i,j) \}) = p_{ij}\,.
\end{equation*}

On the other hand, it is essential to quantify the similarity of data points obtained in $\R^2$ to facilitate a comparison between high and low-dimensional spaces. Let $\mathcal{Y}:=\{ y_i \}_{i=1}^n$ denote a set of $n$ points in $\R^2$. Unlike $p_{ij}$, we utilize the student t-distribution to define the affinity $q_{ij}$ as
\begin{equation}
\label{eqn: def q}
    q_{ij} = \frac{(1+\| y_i - y_j \|^2)^{-1}}{\sum_{k \neq l} (1+ \| y_k - y_l \|^2)^{-1}}\,,
\end{equation}
where $i\neq j$.
Note that $q_{ij}$ is already symmetric, so a symmetrization is not needed. Similarly, affinity $q_{ij}$ defines a probability $Q$ on the set $\{(y_i,y_j)|i\neq j\}\subset \mathbb{R}^2\times \mathbb{R}^2$ by
\begin{equation*}
    Q(\{(i,j)\}) = q_{ij}\,,
\end{equation*}
where $i\neq j$.
The term ``t-distribution'' is derived from the kernel used for the affinity $q_{ij}$. In t-SNE's predecessor, SNE \cite{Hinton_Roweis_2003}, the affinity for embedded data points is quantified using the Gaussian function. In t-SNE, a variant of SNE, the affinity for data points in low dimension is defined using the student t-distribution function, leading to the algorithm's name, t-SNE.

To quantify the similarities between the original high-dimensional point cloud and the embedded point cloud, t-SNE utilizes the \emph{Kullback-Leibler divergence (KL divergence)} \cite{kullback1951information} for the associated probability density functions $P$ and $Q$. Recall that the KL divergence of two probability distributions $P$ and $Q$ is defined as: 
\begin{equation}
\label{eqn: KL div}
    \KL{P}{Q}  = \sum_{i,j| i \neq j } p_{ij} \log \frac{p_{ij}}{q_{ij}}\,.
\end{equation}
We could also view the KL divergence as a function of $\mathcal{Y}$. To emphasize this relationship, denote 
\[
\mathcal C(\mathcal{Y}):=\KL{P}{Q}\,.
\]
The KL divergence is widely recognized for capturing the difference between two probabilities. Specifically, if the KL divergence is small, the two probabilities are deemed similar. Additionally, the KL divergence is non-negative and equals 0 only when the two probabilities are identical. Consequently, t-SNE seeks points in $\R^2$ with affinities $q_{ij}$ that are as close as possible to the affinities $p_{ij}$ of data points in $\R^d$ in the sense of minimizing \eqref{eqn: KL div}. To find the desired points in $\R^2$, given an initial set $\mathcal{Y}^{(0)}$ comprising $n$ points in $\R^2$, t-SNE utilizes gradient descent to find a minimizer of \eqref{eqn: KL div} with the following rule:
\begin{equation}
\label{eqn: t-sne grad}
    \mathcal{Y}^{(t)} = \mathcal{Y}^{(t-1)} + \eta \frac{\delta \mathcal C(\mathcal{Y}^{(t-1)})}{\delta \mathcal{Y}^{(t-1)}} + \alpha \left( \mathcal{Y}^{(t-1)} - \mathcal{Y}^{(t-2)} \right)\,,
\end{equation}
where $\eta>0$ and $\alpha\geq 0$ are the parameters that we choose, and they are called the {\em learning rate} and {\em momentum} respectively. In general, the learning rate and momentum can also depend on time, but we keep them constant in this study. The momentum term is introduced to accelerate optimization at the outset of the algorithm and mitigate potential convergence to poor local minima (refer to Section 2 of \cite{JMLR:v9:vandermaaten08a}). Empirically, iterative updates of the point set $\mathcal{Y}^{(t)}$ with \eqref{eqn: t-sne grad} that decreases the KL divergence leads to a desired set of points. This paper primarily investigates the behavior of this gradient descent in t-SNE, particularly the boundedness of the embedded points and the existence of minimizer of the KL divergence.

\section{Conditions for analysis}\label{sec: cond for analysis}

In this section, we lay down the conditions employed throughout the paper and outline the problem under consideration. We assume that the given high-dimensional data points $x_i$ reside in $\R^d$. Through the utilization of t-SNE, we project these data points $x_i$ into $\R^2$; that is, the embedded points $y_i$ generated by t-SNE are in $\R^2$.

We start with the precise conditions on the given point cloud $\{ x_i \}_{i=1}^n$, and the measure $\mu$ where the points $ x_i$, $i = 1, \cdots, n$ are sampled from independently. The measure $\mu$ is a probability measure in $\Rd$. We assume that there is an $m$-dimensional $C^2$ manifold $\M$ that is isometrically embedded in $\Rd$ such that $\spt{\mu} = \M$; that is, the support of $\mu$ is an $m$-dimensional $C^2$ manifold. Note that we do not assume connectivity of the manifold, and it can have several connected components. The manifold $\M$ is either without boundary or with a Lipschitz boundary. We assume that the measure $\mu$ is absolutely continuous with respect to the $m$ dimensional Hausdorff measure restricted on the manifold $\M$ (or the Riemannian volume measure associated with the induced Riemannian metric from the canonical Euclidean metric in $\Rd$ via the isometric embedding), and the associated density function is bounded away from $0$ and $\infty$. Abusing notations, we may use $\mu(x)$ for the density function of the measure $\mu$ as well. Assume there exist constants $\lmu$ and $\umu$ such that
\begin{equation*}
     0< \lmu < \mu(x) < \umu < \infty, \ \mbox{for all}\ x \in \M.
\end{equation*}
Noting that the manifold $\M$ is assumed to be $C^2$, a $d$-dimensional ball with a small radius intersected with $\M$ is close to a $m$ dimensional Euclidean ball. Then, with the bounds on $\mu$, we have that
\begin{equation}\label{eqn: mu ball r^n}
    \mu(B^d_r(x)) \sim r^m
\end{equation}
for any $x\in \M$ away from $\partial \M$ and small enough $r>0$. If $\partial \M \neq \emptyset$, then it is assumed to be Lipschitz, which implies that the intersection of a $d$ dimensional ball with a small radius centered at a point close to $\partial \M$ intersected with $\M$ contains a subset of $\M$ which is close to $m$ dimensional Euclidean cone where the opening of the cone in decided by the Lipschitz constant. Therefore, we still have \eqref{eqn: mu ball r^n} for any $x\in \M$ that is close to $\partial \M$ and small enough $r>0$. Hence, adjusting the values of $\lmu$ and $\umu$ if necessary, we assume that there exists $R_\mu$ such that if $r < R_\mu$, then
\begin{equation}
    \omega_m \lmu r^m \leq \mu(B^d_r(x)) \leq \omega_m \umu r^m,
\end{equation}
where $\omega_m$ is the volume of an $m$-dimensional unit ball. 
As $\M$ is a $C^2$-Riemannian manifold, at each point $z \in \M$, there exists a Riemannian exponential map, denoted as $\exp_z$, which is defined on a subset of the tangent space $T_z \M$. 
The $C^2$-regularity assumption of $\M$ allows us to differentiate the Riemannian exponential map, and a classic computation shows that $D\exp_z(0) = Id$, where $D$ means the differentiation, under the normal coordinate that we assume in this paper. In particular, $\exp_z$ is Lipschitz in a neighborhood of $z$. Then, the compactness of $\M$ implies that we can obtain a uniform sized ball around each point $z \in \M$ on which $\exp_z$ is uniformly bi-Lipschitz. Taking $R_\mu>0$ smaller if necessary, we also obtain that on $B^m_{R_\mu}(0) \subset T_z\M$,  $\exp_z$ is uniformly bi-Lipschitz with the bi-Lipschitz constant $L>0$; that is,
\begin{equation}\label{eqn: riem exp bi lip}
    \frac{1}{L} \| v_1 - v_2 \| \leq \| \exp_z (v_1) - \exp_z(v_2) \| \leq L \| v_1 - v_2 \|\,,
\end{equation}
for any $v_1, v_2 \in B^m_{R_\mu}(0) \subset T_z \M$.

Next, we discuss some conditions for the sampled points $\{x_i \}_{i=1}^n$. Denote the empirical measure of the set $\{ x_i \}_{i=1}^n$ by $\mu_n$; that is,
\begin{equation*}
    \mu_n = \frac{1}{n} \sum_{i=1}^n \delta_{x_i}.
\end{equation*}
As the points $x_i$, $1\leq i \leq n$ are independently sampled from the probability measure $\mu$, their distribution resembles the probability measure $\mu$ when $n$ is large. In other words, when $n$ is large, $\mu_n$ and $\mu$ become closer in some sense. We quantify this phenomenon by adding an assumption on the 1-Wasserstein distance between $\mu_n$ and $\mu$. 

\begin{Def}
   Let $\mu_0$ and $\mu_1$ be probability measures in $\mathcal{P}(\R^d)$. The \emph{1-Wasserstein} distance between $\mu_0$ and $\mu_1$ is defined by
\begin{equation}\label{eqn: def W1}
    W_1(\mu_0, \mu_1) = \inf_{\gamma \in \Gamma(\mu_0,\mu_1) }\int \| x-y \| d \gamma(x,y)\,,
\end{equation}
where $\Gamma(\mu_0,\mu_1) = \{\gamma \in \mathcal{P}(\Rd \times \Rd) |\, {\Proj{\Rd\times\{0\}}}_\sharp(\gamma)=\mu_0, {\Proj{\{0\}\times \Rd}}_\sharp (\gamma) = \mu_1 \}$. A measure $\gamma \in \Gamma(\mu_0,\mu_1)$ that achieves the minimum of the right hand side of \eqref{eqn: def W1} is called a \emph{Kantorovich solution}.
\end{Def}
It is well-known that a Kantorovich solution always exists (see, for example, \cite[Ch.4]{villani2008optimal}), and we will use this fact later in the proof of Lemma \ref{lem: int sum entropy est}. The $W_1$ distance gives a notion of distance in the probability measure space whenever it is defined. Therefore, we use the $W_1$ distance to quantify how much $\mu_n$ is close to $\mu$. 
Convergence under the $W_1$ distance is equivalent to the weak convergence with the convergence in momentum (See \cite[Ch.6]{villani2008optimal}); that is,
\begin{align*}
    W_1(\mu_n , \mu) \to 0 \Leftrightarrow \mu_n \to \mu \textrm{ weakly and } \int \| x-z \| d\mu_n \to \int \| x-z \| d\mu, \forall z \in \M.
\end{align*}
Since $\M$ is assumed to be compact, $\| x-z \|$ is a bounded continuous function on $\M$ and therefore the weak convergence implies the convergence in momentum $ \int \| x-z \| d\mu_n \to \int \| x-z \| d\mu $. Therefore, convergence in $W_1$ distance is equivalent to the weak convergence.
We assume that there exists a sequence $\epsilon(n)$ such that $\lim_{n \to \infty} \epsilon(n) = 0$ and
\begin{equation}\label{eqn: small w1}
    W_1(\mu_n, \mu) \leq \epsilon(n).
\end{equation}
Under the compactness assumption of $\M$, existence of such $\epsilon$ is equivalent to the weak convergence of $\mu_n$ to $\mu$.

\section{Perplexity}\label{section perp analysis}
In the initial stage of the t-SNE algorithm, we define affinities $p_{ij}$ and $q_{ij}$ using \eqref{eqn: def p} and \eqref{eqn: def q}. A parameter perplexity $Perp$ is employed to determine $\sigma_i$ in \eqref{eqn: def p cond}. In \cite{JMLR:v9:vandermaaten08a}, it is mentioned that the typical selection for $Perp$ falls between 5 and 50. To our knowledge, this empirical selection lacks theoretical backup. In this section we explore perplexity and offer computations to gain a deeper understanding of $Perp$ selection. The result will not only be the foundation of our main focus but also has its own interest. We initiate with a straightforward observation about Shannon entropy.  In \eqref{eqn: def Shannon}, we use $\log_2$ and we raise that to the power of 2 in \eqref{eqn: perp role}. We can change the base as follows.
\begin{equation*}
    H(P_i) = \sum_{j| j \neq i} p_{j|i} \log_2 p_{j|i} = \frac{1}{\log 2} \sum_{j| j \neq i} p_{j|i} \log p_{j|i} =: \frac{1}{\log 2} H_e (P_i).
\end{equation*}
and therefore
\begin{equation*}
    2^{H(P_i)} = 2^{\frac{1}{\log 2}H_e (P_i)} = e^{H_e (P_i)}.
\end{equation*}
Henceforth, by abusing notations, we use $H(P_i)$ for $H_e(P_i)$ and also call it the Shannon entropy in the following. 

Now, we claim that $Perp$ cannot be too big in the following lemmas.
\begin{Lem}
\label{Lem: H limits}
    Consider $\sigma_i$ in \eqref{eqn: def p cond} as an independent variable. Then we have the following two limits.
    \begin{equation*}
        \lim_{\sigma_i \to \infty} H(P_i) = \log(n-1) \textrm{ and } \lim_{\sigma_i \to 0} H(P_i) = \log N_i,
    \end{equation*}
    where $N_i$ is the number of points in $\{x_k\}_{k \neq i}$ that are closest to $x_i$; that is,
    \begin{equation*}
        N_i := | \{ x_j | \| x_i - x_j \| = \min_{k|k \neq i}\| x_k - x_i \| \} |.
    \end{equation*}
\end{Lem}
\begin{proof}
    We can compute the first limit easily. Note that we have
    \begin{align*}
        \lim_{\sigma_i \to \infty} p_{j|i} & = \lim_{\sigma_i \to \infty} \frac{\exp(-\| x_i - x_j\|^2 / 2\sigma_i^2 )}{ \sum_{k|k \neq i} \exp(-\| x_i - x_k \|^2 / 2\sigma_i^2)} \\
        & =  \frac{\lim_{\sigma_i \to \infty} \exp(-  \| x_i - x_j\|^2 / 2\sigma_i^2 )}{\sum_{k| k \neq i} \lim_{\sigma_i \to \infty} \exp(-  \| x_i - x_k \|^2 / 2\sigma_i^2)} \\
        & = \frac{1}{n-1},
    \end{align*}
    and therefore
    \begin{align*}
        \lim_{\sigma_i \to \infty} H(P_i)  & = - \lim_{\sigma_i \to \infty} \sum_{j| j\neq i} p_{j|i} \log p_{j|i} \\
        & = - \sum_{j| j\neq i} \lim_{\sigma_i \to \infty} p_{j|i} \log \lim_{\sigma_i \to \infty} p_{j|i} \\
        & = -\sum_{j|j \neq i} \frac{1}{n-1} \log \frac{1}{n-1} = \log(n-1).
    \end{align*}
    To study $\sigma_i\to 0$, note that
    \begin{align*}
        p_{j|i} & = 
        \frac{\exp(-\| x_i - x_j\|^2 / 2\sigma_i^2 )}{\sum_{k|k \neq i} \exp(-\| x_i - x_k \|^2 / 2\sigma_i^2)} \\
        & = 
        \left(  \sum_{k|k \neq i} \exp \left( (\| x_i - x_j\|^2-\| x_i - x_k \|)^2 / 2\sigma_i^2 \right) \right)^{-1} \label{eqn: lim s to 0 p 1} \numberthis
    \end{align*}
    If $\| x_i - x_j \| > \|x_i - x_k \|$ for some $k$, then \eqref{eqn: lim s to 0 p 1} converges to 0 since one of the summand diverges to $\infty$. Suppose $\| x_i - x_j \| \leq \|x_i - x_k\|$ for all $k$; that is, suppose $x_j$ is the closest point to $x_i$. The summands with the strict inequality converge to 0 and the summands with equality converge to $\exp(0)=1$ when $\sigma_i\to 0$. Thus, $p_{j|i}$ converges to $N_i^{-1}$  when $\sigma_i\to 0$. As a result,
    \begin{align*}
        \lim_{\sigma_i \to 0} H(P_i)  & = - \lim_{\sigma_i \to 0} \sum_j p_{j|i} \log p_{j|i} \\
        & = - \sum_{x_j \textrm{ closest to } x_i} \lim_{\sigma_i \to 0}  p_{j|i} \log p_{j|i} - \sum_{\textrm{else}} \lim_{\sigma_i \to 0}  p_{j|i} \log p_{j|i} \\
        & = - N_i \times N_i^{-1} \log N_i^{-1} - 0 = \log N_i,
    \end{align*}
    where we have used $\lim_{p \to 0} p \log p = 0$ to obtain the third equality.
\end{proof}
Note that if we view $Perp$ as a function of $\sigma_i$, then $Perp$ depends on $\sigma_i$ continuously. Therefore, the above lemma shows that there is a value of $\sigma_i$ which satisfies \eqref{eqn: perp role} if $Perp$ is between $N_i$ and $n-1$. The number $N_i$ depends on the sampled points $\{x_i\}_{i=1}^n$. By the assumptions that we have imposed on the probability measure $\mu$, however, it is easy to see that the probability to have $\| x_i - x_j \| = \| x_i - x_k \|$ for different indexes $i,j,k$ is 0. Therefore, in practice, we can safely assume that $N_i = 1$, and we henceforth use $N_i = 1$ for any $i$.

Since $Perp$ must be chosen between $1$ and $n-1$ by the above lemma, in the next lemma, we show that $\sigma_i$ is {\em uniquely} defined for any $Perp$ between 1 and $n-1$.  Note that it is mentioned in \cite{JMLR:v9:vandermaaten08a} that $Perp$ depends on $\sigma_i$ monotonically, but to our knowledge the proof is lacking. We present the proof here to fill in this gap.

\begin{Lem}\label{Lem: perp monotone}
    Consider $\sigma_i$ in \eqref{eqn: def p cond} as an independent variable. Then the perplexity $Perp$ is a strictly increasing function of $\sigma_i$ as $\sigma_i$ increases on $\sigma_i >0$.
\end{Lem}

\begin{proof}
    Fix $i$. Note that $Perp = e^{H(P_i)}$ and the exponential function is monotone. Therefore, we only need to show that $H(P_i) = -\sum_{j|j\neq i} p_{j|i} \log p_{j|i}$ is a strictly increasing function of $\sigma_i$. To simplify notations, we introduce
    \begin{equation}
    \label{eqn: def D t e}
        \delta_j = \|x_j  -x_i \|^2,\ \ s = \frac{1}{2\sigma_i^2}\ \ \mbox{and}\ \  e_j = \exp(-\delta_j s).
    \end{equation}
    Also, we write $H = H(P_i)$. Then
    \begin{equation*}
    \label{eqn: p=e/sum}
        p_{j|i} = \frac{e_j}{\sum_{k| k \neq i} e_k}.
    \end{equation*}
    Then monotonic increasing of $H$ with respect to $\sigma_i$  as $\sigma_i$ increases is equivalent to the monotonic decreasing of $H$ with respect to $s$ as $s$ increases. To show that the Shannon entropy $H$ is monotonic, we take a derivative of $H$ with respect to $s$.
    \begin{equation*}
        -\frac{d}{ds} H = \frac{d}{ds} \left( \sum_{j|j \neq i} p_{j|i} \log p_{j|i} \right) = \sum_{j|j\neq i} \frac{d p_{j|i}}{ds} \log p_{j|i} + \sum_{j|j\neq i} \frac{d p_{j|i}}{ds}.
    \end{equation*}
    Since $p_{j|i}$ is a probability, $\sum_{j|j \neq i} p_{j|i} = 1$ and $\sum_{j|j\neq i} \frac{d p_{j|i}}{ds} = 0$. Therefore
    \begin{equation*}
        -\frac{d}{ds}H = \sum_{j|j\neq i} \frac{d p_{j|i}}{ds} \log p_{j|i} = \sum_{j|j\neq i} \frac{d p_{j|i}}{ds} \left( \log e_j - \log \sum_k e_k \right) = \sum_{j|j \neq i} \frac{d p_{j|i}}{ds} \log e_j.
    \end{equation*}
    We thus obtain
    \begin{equation*}
        -\frac{d}{ds}H = \sum_{j| j \neq i} \frac{d p_{j|i}}{ds} \log e_j = \sum_{j|j \neq i} \frac{\frac{d e_j}{ds} \sum_{k|k \neq i} e_k - e_j \sum_{k|k \neq i} \frac{d e_k}{ds}}{\left( \sum_{k| k \neq i} e_k \right)^2} \log e_j.
    \end{equation*}
    From \eqref{eqn: def D t e}, we compute $\frac{d e_j}{ds} = -\delta_j e_j$ and 
    \begin{equation*}
        -\frac{d}{ds}H = \frac{1}{\left( \sum_{k|k \neq i} e_k \right)^2} \sum_{j|j \neq i} \left( - \delta_j e_j \sum_{k| k \neq i} e_k + e_j \sum_{k| k \neq i} \delta_k e_k  \right) \log e_j.
    \end{equation*}
    Since our goal is to show the monotonicity, we only need to decide the sign of $\frac{d}{ds}H$, and therefore we consider
    \begin{equation*}
    \label{eqn: dH/dt sign 1}
        \sum_{j| j \neq i} \left( - \delta_j e_j \sum_{k| k \neq i} e_k + e_j \sum_{k| k \neq i} \delta_k e_k  \right) \log e_j.
    \end{equation*}
    Note that $\log e_j = -\delta_j s$. Since $s > 0$, we need to check that 
    \begin{equation}
    \label{eqn: dH/dt sign 2}
         \sum_{j| j \neq i} \delta_j^2 e_j \sum_{k| k \neq i} e_k -\sum_{j| j \neq i} \delta_j e_j \sum_{k| k \neq i} \delta_k e_k
    \end{equation} 
    is signed. Let $u$ and $v$ be the vectors
    \begin{equation*}
        u = \left( \delta_j \sqrt{e_j} \right)_{j \neq i} \textrm{ and } v = \left( \sqrt{e_j} \right)_{j \neq i},
    \end{equation*}
    then by the Cauchy-Schwartz inequality, we have
    \begin{equation}
    \label{eqn: dH/dt sign C-S}
        \left( \sum_{j| j \neq i} \delta_j e_j \right)^2 = \langle u , v \rangle^2 < \| u \|^2 \| v \|^2 = \sum_{j| j \neq i} ( \delta_j \sqrt{e_j} )^2 \sum_{k| k \neq i} (\sqrt{e_k})^2.
    \end{equation}
    \eqref{eqn: dH/dt sign C-S} shows that \eqref{eqn: dH/dt sign 2} is positive. Note that we obtain the strict inequality unless we have $\delta_j = \delta_k$ for all $j,k$; that is, $\| x_j - x_i \| = \|x_k - x_i \|$ for all $j$ and $k$, which cannot happen when $n$ is sufficiently large. Then \eqref{eqn: dH/dt sign 2} is strictly positive, which implies that $\frac{d}{ds}H$ is strictly negative. Then $H$ is a strictly decreasing function of $s$, and we deduce that $H$ is a strictly increasing function of $\sigma_i$.
\end{proof}
Strict monotonicity of $Perp$ with respect to $\sigma_i$, in particular, implies that for a fixed $Perp$, the $\sigma_i$ that satisfies \eqref{eqn: perp role} is unique. Combining with the existence of such $\sigma_i$ that is discussed after Lemma \ref{Lem: H limits}, we obtain the following proposition.
\begin{Prop}
    For any $Perp \in (1,n-1)$, $\sigma_i$ is well-defined almost surely. Moreover, if $Perp >n-1$ or $Perp <1$, then there is no $\sigma_i$ that satisfies \eqref{eqn: perp role}.
\end{Prop}

We discuss one more observation on the perplexity when the number of sampled points $n$ is big. Note that, by assumption \eqref{eqn: small w1}, the empirical measure $\mu_n=\frac{1}{n} \sum_{i=1}^n \delta_{x_i}$ converges to $\mu$ weakly. Then, for any continuous bounded function $f$, we obtain
\begin{equation*}
    \frac{1}{n} \sum_{i =1}^n f(x_i) = \int f(x) d\mu_n \to \int f(x) d\mu(x)
\end{equation*}
as $n \to \infty$. Also, if we miss one point in the sum and use $n-1$ instead of $n$; that is, if we use $\frac{1}{n-1}\sum_{j | j \neq i}$ instead of $\frac{1}{n} \sum_{i=1}^n$, the above convergence is still valid since one point carries a small mass $\frac{1}{n}$ that disappears when $n \to \infty$. We now apply this to the Shannon entropy $H(P_i)$, and have
\begin{align*}
    H(P_i) & = - \sum_{j|j \neq i}\frac{\exp(-\| x_i - x_j\|^2 / 2\sigma_i^2 )}{\sum_{k|k \neq i} \exp(-\| x_i - x_k \|^2 / 2\sigma_i^2)} \log \frac{\exp(-\| x_i - x_j\|^2 / 2\sigma_i^2 )}{\sum_{k|k \neq i} \exp(-\| x_i - x_k \|^2 / 2\sigma_i^2)} \\
    & = - \sum_{j|j \neq i}\frac{\exp(-\| x_i - x_j\|^2 / 2\sigma_i^2 )}{\sum_{k|k \neq i} \exp(-\| x_i - x_k \|^2 / 2\sigma_i^2)} \\
    &\qquad\times\left( \log \exp(-\frac{\| x_i - x_j\|^2}{ 2\sigma_i^2} ) - \log \sum_{k|k \neq i} \exp(-\frac{\| x_i - x_k \|^2}{2\sigma_i^2}) \right) \\
    & = \frac{\sum_{j|j \neq i} \exp(-\| x_i - x_j\|^2 / 2\sigma_i^2 ) \frac{\| x_i - x_j\|^2}{ 2\sigma_i^2}  }{\sum_{k|k \neq i} \exp(-\| x_i - x_k \|^2 / 2\sigma_i^2)} \\
    & \qquad+ \frac{\sum_{j|j \neq i} \exp(-\| x_i - x_j\|^2 / 2\sigma_i^2 ) }{\sum_{k|k \neq i} \exp(-\| x_i - x_k \|^2 / 2\sigma_i^2)} \log \sum_{k|k \neq i} \exp(\frac{-\| x_i - x_k \|^2 }{ 2\sigma_i^2}) \\
    & = \frac{\sum_{j|j \neq i} \exp(-\| x_i - x_j\|^2 / 2\sigma_i^2 ) \frac{\| x_i - x_j\|^2}{ 2\sigma_i^2}  }{\sum_{k|k \neq i} \exp(-\| x_i - x_k \|^2 / 2\sigma_i^2)}  + \log \sum_{k|k \neq i} \exp(\frac{-\| x_i - x_k \|^2 }{ 2\sigma_i^2}) \\
    & = \frac{\frac{1}{n-1}\sum_{j|j \neq i} \exp(-\| x_i - x_j\|^2 / 2\sigma_i^2 ) \frac{\| x_i - x_j\|^2}{ 2\sigma_i^2}  }{\frac{1}{n-1}\sum_{k|k \neq i} \exp(-\| x_i - x_k \|^2 / 2\sigma_i^2)}  + \log \sum_{k|k \neq i} \exp(\frac{-\| x_i - x_k \|^2}{ 2\sigma_i^2}).
\end{align*}
One can observe that the numerator and the denominator of the fist term in the last line converge to their corresponding integrals. However, the last term may not converge to an integral form but diverges as there might be a lot of terms which are close to 1, unless $\sigma_i$ are chosen properly. Hence, to have equation \eqref{eqn: perp role} with a stable value of $\sigma_i$, $Perp$ should change accordingly. For instance, let $0<\zeta<1$ be a constant and chose $Perp = \zeta(n-1)$. Then from \eqref{eqn: perp role}, we observe that
\begin{align*}
    \log \zeta & = H(P_i)-\log (n-1) \\
    & = \frac{\frac{1}{n-1}\sum_{j|j \neq i} \exp(-\| x_i - x_j\|^2 / 2\sigma_i^2 ) \frac{\| x_i - x_j\|^2}{ 2\sigma_i^2}  }{\frac{1}{n-1}\sum_{k|k \neq i} \exp(-\| x_i - x_k \|^2 / 2\sigma_i^2)}  + \log \left(\frac{1}{n-1}\sum_{k|k \neq i} \exp(\frac{-\| x_i - x_k \|^2}{ 2\sigma_i^2})\right) \\
    & \sim \frac{\int_{\M} \exp(-\|x_i - x\|^2/2\sigma_i^2)\frac{\| x_i - x \|^2}{2\sigma_i^2}d\mu(x)}{\int_{\M} \exp(-\|x_i - x\|^2/2\sigma_i^2)d\mu(x)} + \log \int_{\M} \exp(-\|x_i - x\|^2/2\sigma_i^2) d\mu(x)
\end{align*}
when $n$ is sufficiently large.
Then we can expect that the value of $\sigma_i$ will be stable when the number of data points $n$ is sufficiently large. To sum up, to have a ``stable'' $\sigma_i$, we need to choose the perplexity to be proportional to $n-1$; that is, 
\begin{align}
Perp = \zeta(n-1)\label{our selection of perp formula}
\end{align}
for some $0<\zeta<1$. 

In practice, when $n$ is finite but ``big'', for example, of order $10^4$ or $10^5$, we could choose $\zeta$ to be a small constant so that $Perp$ falls in the ``typical range'' between 5 and 50 suggested in \cite{JMLR:v9:vandermaaten08a}. The above argument provides a support for this practical suggestion. However, when $n$ is much larger than $10^5$, the above argument suggests a different range for $Perp$. We leave this practical issue to our future work.

\section{Gradient descent}
We move forward to the gradient descent step of t-SNE. Equation \eqref{eqn: t-sne grad} describes the gradient descent step in the algorithm, with the momentum term $\alpha (\mathcal{Y}^{(t-1)}-\mathcal{Y}^{(t-2)})$ added to mitigate potential poor local minima, enhancing practical outcomes. This paper concentrates on the gradient descent aspect of the algorithm, considering $\mathcal{Y}^{(t)} = \mathcal{Y}^{(t-1)} + \eta \frac{\delta \mathcal C(\mathcal{Y}^{(t-1)})}{\delta\mathcal{Y}^{(t-1)}}$ without the momentum term. Additionally, for theoretical analysis, we employ the gradient flow, the continuous version of gradient descent. Thus, the points $y_i$ in $\R^2$ become functions of $t \geq 0$ satisfying the following gradient flow equation.
\begin{equation}\label{eqn: grad desc conti}
    \frac{d y_i}{dt} = - \nabla_{y_i} \cost(\mathcal{Y}^{(t)})\,,
\end{equation}
where $\Y^{(t)} = \{ y_i(t) \}_{i=1}^n$ is the set of locations of the points $y_i$ at (continuous) time $t$. We will often omit the superscript of $\Y^{(t)}$ whenever there is no confusion. Note that $P$ is fixed. We will call \eqref{eqn: grad desc conti} the {\em continuous gradient descent equation} or just the gradient descent equation. We note here that by the gradient descent equation \eqref{eqn: grad desc conti}, the value of the KL-divergence is a decreasing function of $t$,
\begin{equation}\label{eqn: diff KL div}
    \frac{d}{dt} \cost(\mathcal{Y}) = \sum_{i=1}^n \nabla_{y_i} \cost(\mathcal{Y}) \cdot \frac{dy_i}{dt} = - \sum_{i=1}^n \| \nabla_{y_i} \cost(\mathcal{Y}) \|^2 \leq 0\,. 
\end{equation}
In particular, as the KL-divergence is non-negative, it stays finite. Using that we have the explicit formula \eqref{eqn: def q} we can compute the right hand side of \eqref{eqn: grad desc conti} (this computation can also be found in the appendix of \cite{JMLR:v9:vandermaaten08a}. We provide details for the sake of completeness). We first compute
\begin{equation*}
    \nabla_{y_i} q_{jk} = \left\{ \begin{array}{ll}
        \displaystyle 4q_{jk} \sum_{l|l \neq i} q_{il}  (1+\| y_i - y_l \|^2)^{-1} (y_i - y_l) & j,k \neq l \\
        \displaystyle -2q_{ij} (1+\| y_i - y_j \|^2)^{-1} (y_i - y_j) + 4q_{ij}\sum_{l | l \neq i}q_{il} (1+\| y_i - y_l \|^2)^{-1} (y_i - y_l) & k = i
    \end{array} \right.
\end{equation*}
Then we compute
\begin{align*}
    & -\nabla_{y_i} \cost(\Y)  
     = - \sum_{(j,k)|j \neq k}\frac{d \KL{P}{Q}}{d q_{jk}} \nabla_{y_i} q_{jk} 
     = \sum_{(j,k)|j \neq k} \frac{p_{jk}}{q_{jk}} \nabla_{y_i} q_{jk} \\
    = & \, 4\sum_{(j,k)|j \neq k} p_{jk} \sum_{l|l \neq i}q_{il}  (1+\| y_i - y_l \|^2)^{-1} (y_i - y_l) -4 \sum_{j|j \neq i} p_{ij}(1+\|y_i - y_j \|^2 )^{-1}(y_i - y_j) \\
     = & \,4 \sum_{j|j \neq i}q_{ij}  (1+\| y_i - y_j \|^2)^{-1} (y_i - y_j) -4 \sum_{j|j \neq i} p_{ij}(1+\|y_i - y_j \|^2 )^{-1}(y_i - y_j) \\
     =& \, -4\sum_{j|j\neq i} (p_{ij}-q_{ij})(y_i - y_j)(1+\|y_i - y_j \|^2 )^{-1}.
\end{align*}
Using this formula, we can show the following simple proposition.

\begin{Prop}
\label{prop: center of mass preserve}
    Let $y_i=y_i(t)$, $ 1\leq i \leq n$ be $n$ curves in $\R^2$ that satisfy \eqref{eqn: grad desc conti}. Then the center of mass of $\{ y_i (t) \}_{i=1}^n$ does not change. i.e. 
    \begin{equation*}
        \frac{d}{dt} \sum_{i=1}^n y_i = 0.
    \end{equation*}
\end{Prop}
\begin{proof}
By a direct calculation, we have
    \begin{align*}
        \frac{d}{dt} \sum_{i=1}^n y_i & = -4 \sum_{i=1}^n \sum_{j|j\neq i} (p_{ij}-q_{ij})(y_i - y_k)(1+\|y_i - y_k \|^2 )^{-1} \\
        & = -4 \sum_{i, j | i \neq j}(p_{ij}-q_{ij})(y_i - y_j)(1+\|y_i - y_j \|^2 )^{-1}.
    \end{align*}
    The factors $(p_{ij}-q_{ij})$ and $(1+\| y_i - y_j \|^2)$ are symmetric with respect to $i$ and $j$. The other factor $(y_i - y_j)$, however, is anti-symmetric with respect to $i$ and $j$. Therefore, the last sum above is a symmetric sum of anti-symmetric terms, and hence it is equal to 0.
\end{proof}
Thanks to Proposition \ref{prop: center of mass preserve}, we can fix the center of mass of the set of points $\{y_i\}$ to 0 from now on.

\section{Boundedness of $\{ y_i \}$ and existence of a minimizer}

This is the longest section of this paper showing the first main theorem that the embedding generated by t-SNE is bounded. Through out this section we assume that the embedded points in $\mathbb{R}^2$ can be viewed as curves $y_i(t)$ that satisfy the gradient descent equation \eqref{eqn: grad desc conti} and the center of mass of the set of points $\{y_i(t)\}$ is 0 for all time. To achieve our main theorem, we first observe what happens if t-SNE generates a point that diverges to $\infty$. Then show that if one point diverges to $\infty$, all pairwise distances diverge. Finally, we reach the contradiction and obtain our main theorem.

\subsection{When one point $y_j$ diverges to $\infty$}
We assume that there is a point $y_j = y_j(t)$ that diverges to $\infty$. At this point, however, we do not know that if $y_j$ diverges to $\infty$ in finite time or as $t \to \infty$. Hence, we assume that there is $0<t_\infty \leq \infty$ such that $\lim_{t \nearrow t_\infty} \|y_j(t)\| = \infty$. Without loss of generality, we assume that $j=1$; that is, $\lim_{t \nearrow t_\infty} \|y_1\| = \infty$.

\begin{Lem}\label{lem: diverging pair}
    Let $y_i = y_i(t)$, $1\leq i \leq n$, be n points in $\R^2$ that satisfy \eqref{eqn: grad desc conti}. Suppose $\lim_{t \nearrow t_\infty} y_1 = \infty $. Then there exists another $y_k$ such that 
    \begin{equation*}
        \lim_{t \nearrow t_\infty} \| y_1 - y_k \| = \infty.
    \end{equation*}
\end{Lem}
\begin{proof}
    Suppose, in contrast, that there exists a constant $r>0$ and $t_r>0$ such that $\| y_1(t) - y_i (t)\| < r$ for any $i$ and $t \in [t_r, t_\infty)$. Then the center of mass belongs to $B_r(y_1(t))$ for any $t \in [t_r, t_\infty)$; that is, we have
    \begin{equation}\label{eqn: 0 in intersect ball}
        0 \in \bigcap_{t \in [t_r, t_\infty)} B^2_r(y_1(t))
    \end{equation}
    by the assumption of the center of mass.
    On the other hand, since $y_1 \to \infty$ as $t \nearrow t_\infty$, there exists $s \in (t_r, t_\infty)$ such that $\| y_1 (t_r) - y_1 (s) \| > 2r$. Then we obtain that 
    \begin{equation*}
        \bigcap_{t \in [t_r, t_\infty)} B^2_r(y_1(t)) \subset B^2_r(y_1(t_r)) \cap B^2_r(y_1(s)) = \emptyset,
    \end{equation*} which contradicts to \eqref{eqn: 0 in intersect ball}.
\end{proof}

In the previous lemma, we have seen that if there is a point that diverges to $\infty$, then another point must also exist, such that the distance to the diverging point also diverges to $\infty$. In fact, we can find more points that diverges to $\infty$. 

\begin{Lem}\label{lem: unif dist ratio}
    Suppose $y_1$ diverges to $\infty$ as $t \nearrow t_\infty$. Then, there exists a constant $D>0$ such that for any distinct indexes $i,j,k,l$, 
    \begin{equation}\label{eqn: unif dist ratio}
        \frac{\|y_i - y_j \|}{\| y_k - y_l\|} < D\,,
    \end{equation}
    where $D>0$ is a constant depending on the KL-divergence of the initial embedded points and the input high-dimensional dataset. 
\end{Lem}
\begin{proof}
    By Lemma \ref{lem: diverging pair}, there is another point, say $y_2$, such that $\| y_1 - y_2\|$ diverges to $\infty$ as $t \nearrow t_\infty$. We first show that for any $a\neq b$ that are different from $1$ and $2$, we must have that $\|y_a - y_b \|/ \|y_1 - y_2 \|$ is bounded for any $t$. Indeed, otherwise we obtain
    \begin{align*}
        q_{ab} & \leq \frac{(1+\| y_a - y_b \|^2)^{-1}}{(1+\| y_1 - y_2 \|^2)^{-1}}  = \frac{1+\| y_1 - y_2 \|^2}{1+\| y_a - y_b \|^2} = \frac{\frac{1}{\|y_1 - y_2\|^2}+1}{\frac{1}{\|y_1 - y_2\|^2} + \frac{\| y_a - y_b\|^2}{ \| y_1 - y_2 \|^2}} \to 0,
    \end{align*}
    as $t \nearrow t_\infty$ since $\| y_1 - y_2 \| \to \infty$ and $ \|y_a - y_b \|/ \|y_1 - y_2 \| \to \infty$. Then 
    \begin{align*}
        \cost(\Y) & = \sum_{i,j|i \neq j} p_{ij} \log \frac{p_{ij}}{q_{ij}} \\
        &  \geq \sum_{i,j \left| \substack{i \neq j\\ (i,j)\neq(a,b)}\right.} p_{ij} \log p_{ij} + p_{ab}\log \frac{p_{ab}}{q_{ab}} \to \infty\,,
    \end{align*}
    where we have used $q_{ij} \leq 1$ and $p_{ab}>0$ in the last inequality. The KL-divergence is supposed to stay finite as it is discussed below \eqref{eqn: diff KL div}, and therefore $\cost(\Y) \to \infty$ is a contradiction. Hence, we must have $\|y_a - y_b \|/ \|y_1 - y_2 \| \leq D_{ab}< \infty$ for some $D_{ab}>0$. Next, we claim that we also have $\| y_a - y_b \| / \| y_1 - y_2 \| > D'_{ab}>0$ for some $D'_{ab}$. Suppose, in contrast, that $\| y_a - y_b \| / \| y_1 - y_2 \| \to 0$. Then we obtain
    \begin{align*}
        q_{12} & \leq \frac{(1+\| y_1 - y_2 \|^2)^{-1}}{(1+\|y_a - y_b \|^2)^{-1}}  = \frac{1+\| y_a - y_b \|^2 }{1+\| y_1 - y_2 \|^2}  = \frac{\frac{1}{\| y_1 - y_2 \|^2}+\frac{\| y_a - y_b \|^2}{\| y_1 - y_2\|^2}}{\frac{1}{\| y_1 - y_2 \|^2} + 1} \to 0
    \end{align*}
    as $t \to t_\infty$, where we have used that $\| y_1 - y_2 \| \to \infty$ in the last line. Then we again obtain $\cost(\Y) \to \infty$ which is a contradiction, and we obtain the claim. Now, we obtain that for any indexes $i,j,k,l$ such that $i\neq j$ and $k \neq l$, 
    \begin{equation*}
        \frac{\| y_i - y_j \|}{\| y_k - y_l \|} = \frac{\| y_i - y_j \|}{\| y_1 - y_2 \|} \frac{\| y_1 - y_2 \|}{\| y_k - y_l \|} \leq \frac{D_{ij}}{D'_{kl}}.
    \end{equation*}
    Therefore, by taking $D = \max\{ {D_{ij}}/{D'_{kl}} | i \neq j, k \neq l\}$, we conclude the proof. 
\end{proof}

\begin{Rmk}\label{rmk: explicit ratio}
    Lemma \ref{lem: unif dist ratio} implies that if there is a point $y_i$ that diverges to $\infty$, then all the mutual distance of any pairs of points in $\{ y_i \}_{i=1}^n$ diverges to $\infty$. In particular, there could be at most one point which stays bounded, and all the other points diverges to $\infty$. In this case, we can assume that all the mutual distances are bigger than 1, and under this assumption, we can compute a value of $D$ explicitly. Noting that the gradient descent \eqref{eqn: grad desc conti} makes $\cost(\Y)$ a decreasing function of $t$, the KL-divergence must be smaller than or equal to its initial value $\cost_0 := \cost(\Y^{(0)})$. Fix indexes $a_1,a_2,b_1,b_2$ such that $a_1 \neq b_1$ and $a_2 \neq b_2$, then
    \begin{align*}
        \cost_0 & \geq \cost(\Y) = \sum_{i,j| i \neq j} p_{ij} \log \frac{p_{ij}}{q_{ij}} \\
        & = \sum_{i,j| i \neq j} p_{ij} \log p_{ij} -\sum_{i,j| i \neq j} p_{ij} \log{q_{ij}} \\
        & \geq \sum_{i,j| i \neq j} p_{ij} \log p_{ij} - p_{a_1 b_1} \log q_{a_1 b_1}\,,
    \end{align*}
    where we use $p_{ij}\log q_{ij}<0 $ in the last inequality. Therefore, we see that 
    \begin{equation*}
        \log q_{a_1 b_1} \geq \frac{\sum_{i,j | i \neq j} p_{ij} \log p_{ij} - \cost_0}{p_{a_1 b_1}} \geq \left(\sum_{i,j | i \neq j} p_{ij} \log p_{ij} - \cost_0\right) / \min_{i,j|i \neq j}\{p_{ij}\}.
    \end{equation*}
    Hence we obtain $q_{a_1 b_1} \geq \exp((\sum_{i,j | i \neq j} p_{ij} \log p_{ij} - \cost_0)/ \min_{i,j|i \neq j}\{p_{ij}\} )$. On the other hand, from \eqref{eqn: def q}, we see
    \begin{align*}
        q_{a_1 b_1} & = \frac{(1+\| y_{a_1} - y_{b_1}\|^2)^{-1}}{\sum_{i,j|i \neq j}(1+\| y_{i} - y_{j}\|^2)^{-1} } \\
        & \leq \frac{\|y_{a_1}-y_{b_1}\|^{-2}}{\sum_{i,j|i\neq j} (2\| y_i - y_j \|^2)^{-1}} \\
        & \leq \frac{2\|y_{a_2}-y_{b_2}\|^{2}}{\|y_{a_1}-y_{b_1}\|^{2}},
    \end{align*}
    where we have used that $\| y_i - y_j \| \geq 1$ for any $i \neq j$ by Lemma \ref{lem: unif dist ratio}  to obtain the first inequality. Therefore, we obtain
    \begin{equation*}
        \frac{2\|y_{a_2}-y_{b_2}\|^{2}}{\|y_{a_1}-y_{b_1}\|^{2}} \geq \exp((\sum_{i,j | i \neq j} p_{ij} \log p_{ij} - \cost_0) / \min_{i,j|i \neq j}\{p_{ij}\}).
    \end{equation*}
    Letting 
    \begin{align}
    D := {\sqrt2}\exp\left(\frac{\cost_0-\sum_{i,j | i \neq j} p_{ij} \log p_{ij}}{ 2\min_{i,j|i \neq j}\{p_{ij}\}} \right)
    \end{align}
    and noting that the indexes $a_1,a_2, b_1, b_2$ were arbitrary, we obtain \eqref{eqn: unif dist ratio} with an explicit value of $D$.
\end{Rmk}

\subsection{Information from mutual distances}
The gradient descent equation \eqref{eqn: grad desc conti} outlines the points' behavior over time $t$. Extracting information about the behavior of $y_i$ directly from \eqref{eqn: grad desc conti} is challenging, primarily because the equation depends on all other points $y_j$. Still, with our assumption $\sum_i y_i = 0$, we have
\begin{equation*}
    \| y_i \| = \Big\| y_i - \frac{1}{n} \sum_{j} y_j \Big\| \leq \frac{1}{n} \sum_{j} \| y_i - y_j \|\,.
\end{equation*}
When there is a point $y_i$ that diverges to $\infty$, Lemma \ref{lem: diverging pair} and Lemma \ref{lem: unif dist ratio} jointly imply that all the mutual distances diverge to $\infty$. Hence, we may try to extract information from the mutual distance instead of equation \eqref{eqn: grad desc conti}. Moreover, \eqref{eqn: unif dist ratio} suggests that all the mutual distance diverges with a comparable speed. This motivates us to study the sum of all mutual distance. In the next lemma, we compute the derivative of $\sum_{i,j| i \neq j} \| y_i - y_j \|^2$. We can observe in the proof that the symmetric structure of the affinities gives useful information about the derivative of $\sum_{i,j| i \neq j} \| y_i - y_j \|^2$.

\begin{Lem}
    \label{Lem: diff sum dij}
    \begin{equation}
    \label{eqn: diff sum dij}
        \frac{d}{d t} \sum_{i,j|i \neq j} \| y_i - y_j \|^2  = 24 \sum_{i,j|i \neq j} (p_{ij} - q_{ij}) (1 + \|y_i - y_j\|^2)^{-1}.
    \end{equation}
\end{Lem}
\begin{proof}
    We first compute $\frac{d}{d t} \|y_i - y_j \|^2$.
    \begin{align*}
        \left\langle y_i - y_j, \frac{d}{d t} y_i \right\rangle & = \left\langle y_i - y_j, - 4 \sum_{l| l \neq i} (p_{il} - q_{il})(y_i - y_l)(1+\| y_i - y_l \|^2)^{-1} \right\rangle \\
        & = -4\sum_{l| l \neq i} (p_{il} - q_{il}) \langle y_i - y_j, y_i - y_l  \rangle (1+\| y_i - y_l \|^2)^{-1}\,.
    \end{align*}
    Therefore,
    \begin{align}
        \frac{d}{d t} \| y_i - y_j \|^2 & = 2\left\langle y_i - y_j, \frac{d}{d t}y_i \right\rangle + 2\left\langle y_j - y_i , \frac{d}{d t}y_j \right\rangle \label{eqn: diff dij}\\
        & = -8\sum_{l| l \neq i} (p_{il} - q_{il}) \langle y_i - y_j, y_i - y_l  \rangle (1+\| y_i - y_l \|^2)^{-1} \nonumber \\
        & \ \  -8\sum_{l| l \neq j} (p_{jl} - q_{jl}) \langle y_j - y_i, y_j - y_l  \rangle (1+\| y_j - y_l \|^2)^{-1}\,.\nonumber
    \end{align}
    Now we first consider the terms $l=j$ in the first sum and $l=i$ in the second sum. We have
    \begin{align*}
        & \ \ -\frac{8(p_{ij} - q_{ij})}{1+\| y_i - y_j \|^2} \langle y_i - y_j, y_i - y_j \rangle -\frac{8(p_{ji} - q_{ji})}{1+ \| y_j - y_i \|^2} \langle y_j - y_i, y_j - y_i \rangle \\
        & =  -16 (p_{ij} - q_{ij}) \frac{\| y_i - y_j \|^2}{1+\|y_i - y_j \|^2}.
    \end{align*}
    Next, we consider terms $l=k$ in the first and second sum in \eqref{eqn: diff dij}, which become
    \begin{equation}
    \label{eqn: k term in diff dij}
        -\frac{ 8(p_{ik} - q_{ik})}{1+\| y_i - y_k \|^2} \langle y_i - y_j, y_i - y_k  \rangle -\frac{ 8(p_{jk} - q_{jk})}{1+\| y_j - y_k \|^2} \langle y_j - y_i, y_j - y_k \rangle.
    \end{equation}
    These terms do not simplify much, but we can combine these terms with other terms with different indexes. We pick the terms that contain $i$ in the indexes from $\frac{d}{dt} \| y_j - y_k \|^2$, and the terms that contain $j$ in the indexes from $\frac{d}{dt} \|y_k - y_i \|^2$, and add them to \eqref{eqn: k term in diff dij}:
    \begin{align*}
        -\frac{ 8(p_{ik} - q_{ik})}{1+\| y_i - y_k \|^2} \langle y_i - y_j, y_i - y_k  \rangle -\frac{ 8(p_{jk} - q_{jk})}{1+\| y_j - y_k \|^2} \langle y_j - y_i, y_j - y_k \rangle \\
        -\frac{ 8(p_{ji} - q_{ji})}{1+\| y_j - y_i \|^2} \langle y_j - y_k, y_j - y_i  \rangle -\frac{ 8(p_{ki} - q_{ki})}{1+\| y_k - y_i \|^2} \langle y_k - y_j, y_k - y_i \rangle \\
        -\frac{ 8(p_{kj} - q_{kj})}{1+\| y_k - y_j \|^2} \langle y_k - y_i, y_k - y_j  \rangle -\frac{ 8(p_{ij} - q_{ij})}{1+\| y_i - y_j \|^2} \langle y_i - y_k, y_i - y_j \rangle.
    \end{align*}
    Note that we can combine the first term in the first line and the second term in the second line to obtain
    \begin{align*}
        & \ \ -\frac{ 8(p_{ik} - q_{ik})}{1+\| y_i - y_k \|^2} \langle y_i - y_j, y_i - y_k  \rangle -\frac{ 8(p_{ki} - q_{ki})}{1+\| y_k - y_i \|^2} \langle y_k - y_j, y_k - y_i \rangle \\
        & = -\frac{ 8(p_{ik} - q_{ik})}{1+\| y_k - y_i \|^2} \langle -y_i + y_j + y_k - y_j, y_k - y_i \rangle \\
        & = -8(p_{ik} - q_{ik}) \frac{\| y_k - y_i \|^2}{1+\|y_k - y_i \|^2}\,.
    \end{align*}
    We can do a similar computation with the first term in the second line and the second term in the third line, and with the first term in the third line and the second term in the first line. By combining all these computations, we obtain 
    \begin{equation}
    \label{eqn: diff sum dij 2}
        \frac{d}{dt} \sum_{ij} \| y_i - y_j \|^2 = -24 \sum_{ij} (p_{ij}-q_{ij}) \frac{\|y_i - y_j \|^2}{1 + \| y_i - y_j \|^2 }\,.
    \end{equation}
    Noting that both $p_{ij}$ and $q_{ij}$ represent probability, we have
    \begin{align*}
        \sum_{ij} (p_{ij} - q_{ij} ) \frac{\|y_i - y_j \|^2}{1 + \| y_i - y_j \|^2 } & =\sum_{ij} (p_{ij} - q_{ij} ) \frac{1+\|y_i - y_j \|^2-1}{1 + \| y_i - y_j \|^2 } \\
        & = \sum_{ij} (p_{ij} - q_{ij}) - \sum_{ij} (p_{ij} - q_{ij})(1+\|y_i - y_j\|^2)^{-1} \\
        & = -\sum_{ij} (p_{ij}-q_{ij})(1+\| y_i - y_j \|^2)^{-1}.
    \end{align*}
    We apply this to \eqref{eqn: diff sum dij 2} to obtain the desired result.
\end{proof}

\begin{Rmk}\label{rmk: cauchy schwartz}
    The formula \eqref{eqn: diff sum dij} can be changed as follows by the definition of $q_{ij}$:
    \begin{equation*}
        24\sum_{i,j| i \neq j}(p_{ij}-q_{ij})(1+\| y_i - y_j \|^2)^{-1} = 24 \sum_{i,j| i \neq j} (p_{ij}-q_{ij})q_{ij} \sum_{k, l| k \neq l} (1+\| y_k - y_l \|^2)^{-1}.
    \end{equation*}
    Noting that $ \sum_{k, l| k \neq l} (1+\| y_k - y_l \|^2)^{-1}$ is always positive, the sign of $\frac{d}{dt} \| y_i - y_j \|^2$ is decided by the sign of $\sum_{i,j| i \neq j} (p_{ij}-q_{ij})q_{ij}$. Moreover, noting that $p_{ij}q_{ij} \leq \frac{1}{2} (p_{ij}^2 + q_{ij}^2) $, we obtain
    \begin{equation*}
        \sum_{i,j| i \neq j} (p_{ij}-q_{ij})q_{ij} \leq \frac{1}{2}\sum_{i,j | i \neq j} (p_{ij}^2 - q_{ij}^2),
    \end{equation*}
    and hence
    \begin{equation}\label{eqn: diff sum dij up bound}
        \frac{d}{dt} \sum_{i,j| i \neq j} \| y_i - y_j \|^2 \leq 12 \sum_{i,j | i \neq j} (p_{ij}^2 - q_{ij}^2) \sum_{k,l|k \neq l} (1+\|y_k - y_l \|^2)^{-1}.
    \end{equation}
    In particular, if $\sum_{i,j|i\neq j}p_{ij}^2 < \sum_{i,j|i\neq j} q_{ij}^2$, then $\frac{d}{dt} \sum_{i,j| i \neq j} \| y_i - y_j \|^2$ is negative; that is, $\sum_{i,j| i \neq j} \| y_i - y_j \|^2$ is decreasing.
\end{Rmk}

\begin{Rmk}
    In Lemma \ref{Lem: diff sum dij}, we used \eqref{eqn: grad desc conti} to compute the derivative of distance $\| y_i - y_j \|$. Recall that in the original t-SNE case, the gradient descent is discrete \eqref{eqn: t-sne grad}. In this case, we can do a similar computation with a difference quotient instead of derivative. Using equation \eqref{eqn: t-sne grad} without the momentum term (i.e. $\alpha(t) = 0$), we obtain
    \begin{align*}
        & \frac{1}{\eta} \left( \| y_i(t) - y_j(t) \|^2 - \| y_i(t-1) - y_j(t-1) \| \right) \\
        = & \left\langle \frac{1}{\eta} (y_i(t) - y_i(t-1) ) - \frac{1}{\eta} (y_j (t) - y_j(t-1) ), (y_i(t) - y_j (t) ) + (y_i(t-1)-y_j(t-1)) \right\rangle \\
        = & \left\langle - \nabla_{y_i} \cost(\Y^{t-1})+\nabla_{y_j} \cost(\Y^{(t)}), y_i(t) + y_j (t) \right\rangle \\
        & + \left\langle - \nabla_{y_i} \cost(\Y^{(t-1)})+\nabla_{y_j} \cost(\Y^{(t-1)}), y_i(t-1) - y_j (t-1) \right\rangle.
    \end{align*}
    We can apply the exactly same computation to the second term in the last line of the above equation, and obtain an equation that is analogous to \eqref{eqn: diff sum dij 2}.
    \begin{align*}
        & \left\langle - \nabla_{y_i} \cost(\Y^{(t-1)})+\nabla_{y_j} \cost(\Y^{(t-1)}), y_i(t-1) - y_j (t-1) \right\rangle \\
        = & - 12 \sum_{i,j|i \neq j} (p_{ij}(t-1) - q_{ij}(t-1)) \frac{\| y_i(t-1) - y_j(t-1)\|^2}{(1+\| y_i (t-1) - y_j(t-1) \|^2))}.
    \end{align*} 
    Almost the same computation applies to the first term. However, we obtain the following equation which is similar to, but different from \eqref{eqn: diff sum dij 2}:
    \begin{align*}
        & \left\langle - \nabla_{y_i} \cost(\Y^{t-1})+\nabla_{y_j} \cost(\Y^{(t)}), y_i(t) + y_j (t) \right\rangle \\
        = & - 12 \sum_{i,j|i \neq j} (p_{ij}(t-1) - q_{ij}(t-1)) \frac{\left\langle y_i(t) - y_j(t), y_i(t-1) - y_j(t-1) \right\rangle}{(1+\| y_i (t-1) - y_j(t-1) \|^2))}.
    \end{align*}
    This is due to that the difference quotient involves two different times, and therefore we obtain an equation that involves two different times. In this case, the argument from Remark \ref{rmk: cauchy schwartz} cannot be applied unless we have $y_i (t-1) - y_j(t-1) \approx y_i(t) - y_j(t)$ in some sense. Since our focus in this paper is the continuous setup, this topic will be explored in our future work.
\end{Rmk}

When all the distances $\| y_i - y_j \|$ diverges to $\infty$, $1$ is dominated by $\| y_i - y_j \|^2$, and hence the affinity $q_{ij}$ can be approximately computed as
\begin{equation}
    q_{ij} = \frac{(1+\|y_i - y_j \|^2)^{-1}}{\sum_{k,l | k \neq l }(1+\|y_k - y_l \|^2)^{-1}} \sim \frac{\| y_i - y_j \|^{-2}}{\sum_{k,l | k \neq l}\| y_k - y_l \|^{-2}} =: q'_{ij}\,.\label{definition q'ij}
\end{equation}
Indeed, if all mutual distances are bigger than 1; that is, $\| y_i - y_j \|> 1$ for all $i\neq j$, then we have
\begin{equation*}
    q_{ij}=\frac{(1+\|y_i - y_j \|^2)^{-1}}{\sum_{k,l | k \neq l }(1+\|y_k - y_l \|^2)^{-1}} \leq \frac{\|y_i - y_j \|^{-2}}{\sum_{k,l | k \neq l }(2\|y_k - y_l \|^2)^{-1}} = 2q'_{ij},
\end{equation*}
and
\begin{equation*}
   q_{ij}=\frac{(1+\|y_i - y_j \|^2)^{-1}}{\sum_{k,l | k \neq l }(1+\|y_k - y_l \|^2)^{-1}} \geq \frac{2^{-1}\|y_i - y_j \|^{-2}}{\sum_{k,l | k \neq l }\|y_k - y_l \|^{-2}} = \frac{1}{2} q'_{ij}.
\end{equation*}
Applying this to \eqref{eqn: diff sum dij up bound}, we obtain that
\begin{equation}\label{eqn: diff sum up bound'}
    \frac{d}{dt} \sum_{i,j| i \neq j} \| y_i - y_j \|^2 \leq 12 \sum_{i,j | i \neq j} (p_{ij}^2 - \frac{1}{2} {q'_{ij}}^2) \sum_{k,l|k \neq l} (1+\|y_k - y_l \|^2)^{-1}
\end{equation}
provided when $\| y_i - y_j \|>1$ for any $i\neq j$.

Based on the above discussion, if the set of points $\{y_i\}$ stays bounded regardless of its initial state $\{ y_i (0) \}$, we expect to have $ \frac{d}{dt}\sum_{i,j|i \neq j} \| y_i - y_j \|^2<0$ when all $\| y_i - y_j \|$ are too big. Equation \eqref{eqn: diff sum up bound'} suggests that if $\sum_{i,j| i \neq j} p_{ij}^2 < \frac{1}{2}\sum_{i,j| i \neq j} {q_{ij}'}^2$, we can obtain that $\frac{d}{dt} \sum_{i,j|i \neq j} \| y_i - y_j \|^2<0$ when all the mutual distances $\| y_i - y_j\|$ are greater than 1. To continue the exploration of the boundedness property of the embedded points, we need to better understand the affinity $p_{ij}$ and the behavior of embedded points via $q_{ij}'$ assuming divergence.

\subsection{Affinities $p_{ij}$ and $q_{ij}'$}
We first consider the affinity $p_{ij}$. Noting that $p_{ij}$ is defined as the symmetric sum of conditional affinities $p_{j|i}$, we should look at the formula \eqref{eqn: def p cond}. Taking log on the ratio of two conditional affinities $p_{j|i}$ and $p_{k|i}$, we obtain
\begin{equation}\label{eqn: log ratio p}
\log \frac{p_{j|i}}{p_{k|i}} = \frac{1}{2\sigma_i^2}  (\| x_k - x_i \|^2 - \| x_j - x_i \|^2) \leq \frac{\diam{\M}^2}{\sigma_i^2}.
\end{equation}
On the other hand, from Lemma \ref{Lem: H limits} with an assumption $Perp = \zeta(n-1)$, where $\zeta\in (0,1)$, we have that $\sigma_i > 0$ almost surely. If we can have a uniform lower bound of $\sigma_i$ that does not depend on $n$, then we can deduce from \eqref{eqn: log ratio p} that $p_{j|i} \sim p_{k|i}$. To obtain a uniform lower bound of $\sigma_i$ that does not depend on $n$, we need an equation about $\sigma_i$ that does not depend on $n$. 

\begin{Lem}\label{lem: int sum entropy est}
    Fix $\sigma>0$ and let $H(P_i)$ be the Shannon entropy \eqref{eqn: def Shannon} where we replaced $\sigma_i$ in \eqref{eqn: def p cond} by $\sigma$. Then 
    \begin{align*}
        & \underline{M} ( H(P_i) - \log n ) + \underline{E}  \\
        \leq & \,- \int \frac{\exp(-\| x - x_i \|^2 / 2\sigma^2)}{\int \exp(-\| \bar{x} - x_i \|^2 / 2\sigma^2) d\mu(\bar{x})} \log \frac{\exp(-\| x - x_i \|^2 / 2\sigma^2)}{\int \exp(-\| \bar{x} - x_i \|^2 / 2\sigma^2) d\mu(\bar{x})} d\mu(x) \\
        \leq & \,\overline{M} (H(P_i) - \log n) + \overline{E},
    \end{align*}
    for some constants $\overline{M}$, $\underline{M}$, $\overline{E}$, and $\underline{E}$ that depend on $n$ and $\sigma$. For the fixed $\sigma$, we have
    \begin{equation*}
        \lim_{n \to \infty}\overline{M} = \lim_{n \to \infty} \underline{M}=1\ \ \mbox{and}\ \ \lim_{n \to \infty} \overline{E} = \lim_{n \to \infty} \underline{E} = 0.
    \end{equation*}
\end{Lem}
\begin{proof}
    To simplify the notation, we use $s = \frac{1}{2\sigma^2}$. We first note that
    \begin{align*}
       & - \int \frac{\exp(-\| x - x_i \|^2 s)}{\int \exp(-\| \bar{x} - x_i \|^2 s) d\mu(\bar{x})} \log \frac{\exp(-\| x - x_i \|^2 s)}{\int \exp(-\| \bar{x} - x_i \|^2 s) d\mu(\bar{x})} d\mu(x) \\
        = & \, \frac{\int \exp(-\| x - x_i \|^2 s) \| x - x_i \|^2 s d\mu(x)}{\int \exp(-\| {x} - x_i \|^2 s) d\mu({x})} + \log \int \exp(-\| {x} - x_i \|^2 s) d\mu({x})\,.
    \end{align*}
    We also note that we can estimate the sum $\sum_{j|j \neq i}$ using the empirical measure $\mu_n$; that is,
    \begin{equation}
        \frac{1}{n}\sum_{j|j \neq i} \exp(-\| x_j - x_i \|^2s) = \int \exp(-\| x - x_i \|^2 s) d\mu_n(x) - \frac{1}{n}\,. \label{proof 6.7 approximation using mu_n}
    \end{equation}
    In particular, we have
    \begin{equation}\label{eqn: sum mu_n comp 1}
        \frac{1}{n}\sum_{j|j \neq i} \exp(-\| x_j - x_i \|^2s) \leq \int \exp(-\| x - x_i \|^2 s) d\mu_n(x)\,.
    \end{equation}
    In addition, we observe
    \begin{align*}
        &\frac{1}{\frac{1}{n} \sum_{j|j \neq i} \exp(-\| x_j - x_i \|^2s)} - \frac{1}{\frac{1}{n} \sum_{j|j \neq i} \exp(-\| x_j - x_i \|^2s) + \frac{1}{n}} \\
        =&\, \frac{1}{n} \times \frac{1}{\frac{1}{n} \sum_{j|j \neq i} \exp(-\| x_j - x_i \|^2s) \left( \frac{1}{n} \sum_{j|j \neq i} \exp(-\| x_j - x_i \|^2s) + \frac{1}{n} \right)} \\
        \leq &\, \frac{1}{n} \times \frac{1}{\left( \frac{1}{n} \sum_{j|j \neq i} \exp(-\| x_j - x_i \|^2s) \right)^2} \\
        \leq &\, \frac{1}{n} \times \frac{1}{\left( \frac{n-1}{n} \exp(-\diam{\M}^2s)\right)^2}  \\
        =& \,\frac{1}{n-1} \exp(2\diam{\M}^2s)\,,
    \end{align*}
    which combined with \eqref{proof 6.7 approximation using mu_n} lead to
    \begin{align}\label{eqn: sum mu_n comp 2}
        &\frac{1}{\int \exp(-\| x - x_i \|^2s)d\mu_n } \\
        \geq&\, \frac{1}{\frac{1}{n} \sum_{j|j \neq i} \exp(-\| x_j - x_i \|^2s)} - \frac{1}{n-1} \exp(2\diam{\M}^2s)\,.\nonumber
    \end{align}
    Also, we have
    \begin{equation}\label{eqn: sum mu_n comp 3}
        \frac{1}{n}\sum_{j|j \neq i} \exp(-\| x_j - x_i \|^2s) \| x_j - x_i \|^2s = \int  \exp(-\| x - x_i \|^2 s) \| x- x_i \|^2 s d\mu_n.
    \end{equation}
    Let $\gamma_n$ be a Kantorovich solution to the optimal transportation problem with the Euclidean distance cost from $\mu$ to $\mu_n$ so that
    \begin{equation}\label{eqn: gamma n}
        \int \| x - \bar{x} \| d\gamma_n(x,\bar{x}) = \inf_{\gamma \in \Gamma(\mu,\mu_n)} \int \| x-\bar{x} \| d\gamma(x,\bar{x})\,.
    \end{equation}
    Then, noting that the exponential function is 1-Lipschitz on negative numbers,
    \begin{align}
       & \left| \int \exp(-\| x - x_i \|^2 s) d\mu(x) - \int \exp(-\| x - x_i \|^2s)d\mu_n(x) \right|\nonumber \\
       =&\, \left| \int \left[ \exp(-\| x - x_i \|^2 s) - \exp(-\| \bar{x} - x_i \|^2 s) \right]d\gamma_n(x,\bar{x})\right|\nonumber \\
       \leq &\, s \int \left|- \| x - x_i \|^2  + \| \bar{x} - x_i \|^2  \right|d\gamma_n(x,\bar{x}) \nonumber \\
       = & \, s\int \left| \langle x+\bar{x} - 2x_i , \bar{x}-x \rangle \right|  d\gamma_n(x,\bar{x})\nonumber\\
       \leq &\, 2\diam{\M}s \int \| x- \bar{x} \| d\gamma_n(x,\bar{x})\,,\label{proof lemma 6.7 bound 1}
    \end{align}
    where the last inequality comes from the Cauchy-Schwartz inequality and the last integral is the 1-Wasserstein distance since $\gamma_n$ is a Kantorovich solution. Again, since $\gamma_n$ is a Kantorovich solution, we can use the assumption \eqref{eqn: small w1}. Continuing,
    \begin{align}
        2\diam{\M}s \epsilon(n) & = 2\diam{\M}s \epsilon(n) \exp(\diam{M}^2 s) \exp(-\diam{\M}^2s) \label{proof lemma 6.7 bound 2}\\
        & \leq 2\diam{\M}s \exp(\diam{\M}^2s)  \epsilon(n) \int \exp(-\| x-x_i \|^2 s) d\mu\,.\nonumber
    \end{align}
    In the last inequality, we have used that $\| x - x_i \| \leq \diam{\M}$ and that $\mu$ is a probability measure. Letting $M_1 =M_1(n,s)= 2\diam{\M}s \exp(\diam{\M}^2s)  \epsilon(n)$, with \eqref{proof lemma 6.7 bound 1} and \eqref{proof lemma 6.7 bound 2} we obtain
\begin{align}\label{eqn: mu_n mu comp 1}
        (1-M_1) \int \exp(-\| x - x_i \|^2 s) d\mu(x)\,& \leq \int \exp(-\| x - x_i \|^2s)d\mu_n(x) \\
        &\leq (1+M_1) \int \exp(-\| x - x_i \|^2 s) d\mu(x)\nonumber
    \end{align}
     and hence
    \begin{align*}
        \log \int \exp(-\| x - x_i \|^2 s) d\mu_n(x) \leq \log \int \exp(-\| x - x_i \|^2 s) d\mu(x) + \log (1+M_1)\,.
    \end{align*}
    Also, noting that $te^{-t}$ is Lipschitz on positive numbers with the Lipschitz constant bounded by 1, we have
    \begin{align*}\label{eqn: mu_n mu comp 2}
        & \left| \int \exp(-\| x-x_i \|^2 s)\| x-x_i \|^2 s d\mu(x) - \int \exp(-\| x - x_i \|^2s)\| x-x_i \|^2s d\mu_n(x)\right| \\
        = & \,\left| \int \left( \exp(-\| x-x_i \|^2 s)\| x- x_i \|^2s - \exp(-\| \bar{x} - x_i \|^2s)\| \bar{x} - x_i \|^2 s \right) d\gamma_n(x,\bar{x}) \right| \\
        \leq & \,\int \left| -\| x-x_i \|^2 s + \| \bar{x} - x_i \|^2s \right| d\gamma_n(x,\bar{x}) \\
        \leq & \,M_1 \int \exp(-\| x - x_i \|^2 s) d\mu(x). \numberthis
    \end{align*}
    Then we obtain
    \begin{align*}
        & \frac{\int \exp(-\| x - x_i \|^2s)\| x-x_i \|^2s d\mu_n(x)}{\int \exp(-\| x - x_i \|^2 s) d\mu_n(x)} \\
        \leq & \,\frac{\int \exp(-\| x-x_i \|^2 s)\| x-x_i \|^2 s d\mu(x) + M_1 \int \exp(-\| x - x_i \|^2 s) d\mu(x) }{ (1-M_1) \int \exp(-\| x - x_i \|^2 s) d\mu(x)} \\
        = & \,\frac{1}{1-M_1} \frac{\int \exp(-\| x-x_i \|^2 s)\| x-x_i \|^2 s d\mu(x)}{\int \exp(-\| x - x_i \|^2 s) d\mu(x)} + \frac{M_1}{1-M_1}\,.
    \end{align*}
    By rearranging some terms and combining the above, we obtain
    \begin{align*}
        & \frac{\int \exp(-\| x - x_i \|^2s)\| x-x_i \|^2s d\mu(x)}{\int \exp(-\| x - x_i \|^2 s) d\mu(x)} + \log \int \exp(-\| x - x_i \|^2 s) d\mu(x) \\
        \geq &\, (1-M_1) \frac{\int \exp(-\| x - x_i \|^2s)\| x-x_i \|^2s d\mu_n(x)}{\int \exp(-\| x - x_i \|^2 s) d\mu_n(x)} + \log \int \exp(-\| x - x_i \|^2 s) d\mu_n(x) \\
        & -M_1 - \log(1+M_1) \\
        = & \,(1-M_1) \left( \frac{\int \exp(-\| x - x_i \|^2s)\| x-x_i \|^2s d\mu_n(x)}{\int \exp(-\| x - x_i \|^2 s) d\mu_n(x)} + \log \int \exp(-\| x - x_i \|^2 s) d\mu_n(x) \right) \\
        & -M_1 - \log(1+M_1) + M_1 \log \int \exp(-\| x - x_i \|^2 s) d\mu_n(x) \\
        \geq & \,(1-M_1) \left( \frac{\int \exp(-\| x - x_i \|^2s)\| x-x_i \|^2s d\mu_n(x)}{\int \exp(-\| x - x_i \|^2 s) d\mu_n(x)} + \log \int \exp(-\| x - x_i \|^2 s) d\mu_n(x) \right) \\
        & -M_1 - \log(1+M_1) - M_1\diam{\M}^2s\,.
    \end{align*}
    Finally, we estimate the integrals with summations. We use \eqref{eqn: sum mu_n comp 1} on the integral in $\log$, \eqref{eqn: sum mu_n comp 2} on the integral on the denominator, and \eqref{eqn: sum mu_n comp 3} on the integral on the numerator to obtain
    \begin{align*}
         & \frac{\int \exp(-\| x - x_i \|^2s)\| x-x_i \|^2s d\mu_n(x)}{\int \exp(-\| x - x_i \|^2 s) d\mu_n(x)} + \log \int \exp(-\| x - x_i \|^2 s) d\mu_n(x) \\
         \geq &\, \frac{\frac{1}{n} \sum_{j | j \neq i} \exp(-\| x_j - x_i \|^2 s ) \| x_j - x_i \|^2 s }{ \frac{1}{n} \sum_{j | j \neq i} \exp(-\| x_j - x_i \|^2 s ) +\frac{1}{n}} + \log \frac{1}{n} \sum_{j | j \neq i } \exp(-\| x_j - x_i \|^2 s ) \\
         \geq &\, \frac{ \sum_{j | j \neq i} \exp(-\| x_j - x_i \|^2 s ) \| x_j - x_i \|^2 s }{ \sum_{j | j \neq i} \exp(-\| x_j - x_i \|^2 s ) } - \frac{1}{n-1} \diam{\M}^2 s \exp(2\diam{\M}s)\\
         & + \log \sum_{j | j \neq i } \exp(-\| x_j - x_i \|^2 s ) +\log \frac{1}{n} \\
         = &\, H(P_i) - \log n - \frac{1}{n-1}\diam{\M}^2 s \exp(2\diam{\M}s)\,.
    \end{align*}
    Hence, we obtain 
    \begin{align*}
        &\underline{M} ( H(P_i) - \log n ) + \underline{E} \\
        \leq &\,- \int \frac{\exp(-\| x - x_i \|^2 / 2\sigma^2)}{\int \exp(-\| \bar{x} - x_i \|^2 / 2\sigma^2) d\mu(\bar{x})} \log \frac{\exp(-\| x - x_i \|^2 / 2\sigma^2)}{\int \exp(-\| \bar{x} - x_i \|^2 / 2\sigma^2) d\mu(\bar{x})} d\mu(x)\,,
    \end{align*}
    where 
    \begin{equation*}
        \underline{M} = (1- M_1) 
    \end{equation*}
    and
    \begin{equation*}
        \underline{E} = -M_1 - \log(1+M_1) - M_1 \diam{\M}^2s -\frac{1-M_1}{n-1}\diam{\M}^2 s \exp(2\diam{\M}s)\,.
    \end{equation*}
    The proof for the other side bound is similar. We use \eqref{eqn: mu_n mu comp 1} and \eqref{eqn: mu_n mu comp 2} to obtain
    \begin{align*}
        & \frac{\int \exp(-\| x - x_i \|^2 s) \| x - x_i \|^2 s d\mu(x)}{\int \exp(-\| x - x_i \|^2s)d\mu(x)} + \log \int \exp(-\| x - x_i \|^2s)d\mu(x) \\
        \leq &\, (1+M_1) \left( \frac{\int \exp(-\| x - x_i \|^2 s) \| x - x_i \|^2 s d\mu_n(x)}{\int \exp(-\| x - x_i \|^2s)d\mu_n(x)} + \log \int \exp(-\| x - x_i \|^2s)d\mu_n(x) \right) \\
        & + M_1 -\log(1-M_1) - M_1 \log \int \exp(-\| x - x_i \|^2s)d\mu_n(x) \\
        \leq &\, (1+M_1) \left( \frac{\int \exp(-\| x - x_i \|^2 s) \| x - x_i \|^2 s d\mu_n(x)}{\int \exp(-\| x - x_i \|^2s)d\mu_n(x)} + \log \int \exp(-\| x - x_i \|^2s)d\mu_n(x) \right) \\
        & + M_1 -\log(1-M_1) +M_1 \exp(\diam{\M}^2s).
    \end{align*}
    To exchange $\int d\mu_n$ and $\sum_{j|j \neq i}$, we observe that
    \begin{align*}\label{eqn: sum mu_n comp 4}
        & \log \int \exp(-\| x - x_i \|^2s)d\mu_n(x)  = \log \left( \frac{1}{n} \sum_{j|j \neq i}\exp(-\| x_j - x_i \|^2s)+\frac{1}{n} \right)\\
        \leq & \left( \frac{1}{n} \sum_{j | j \neq i} \exp(-\|x_j - x_i \|^2s) \right)^{-1} \frac{1}{n} +\log \frac{1}{n} \sum_{j | j \neq i} \exp(-\|x_j - x_i \|^2s). \numberthis
    \end{align*}
    Note that we have used concavity of $\log$ in the form that the tangent function of $\log$ at $\frac{1}{n} \sum_{j | j \neq i} \exp(-\|x_j - x_i \|^2s)$ is greater than $\log$. Then we estimate the integrals with summations. By using \eqref{eqn: sum mu_n comp 1} on the integral on the denominator, \eqref{eqn: sum mu_n comp 3} on the integral on the numerator, and \eqref{eqn: sum mu_n comp 4} to estimate the $\log$ part, we obtain
    \begin{align*}
        & \frac{\int \exp(-\| x - x_i \|^2 s) \| x - x_i \|^2 s d\mu_n(x)}{\int \exp(-\| x - x_i \|^2s)d\mu_n(x)} + \log \int \exp(-\| x - x_i \|^2s)d\mu_n(x) \\
        \leq &\, \frac{\frac{1}{n}\sum_{j | j \neq i} \exp(-\|x_j - x_i \|^2s) \| x_j - x_i \|^2s }{\frac{1}{n}\sum_{j|j \neq i} \exp(-\|x_j - x_i \|^2s)} + \log \frac{1}{n} \sum_{j|j \neq i} \exp(-\|x_j - x_i \|^2s) \\
        & + \left( \frac{1}{n} \sum_{j | j \neq i} \exp(-\|x_j - x_i \|^2s) \right)^{-1} \frac{1}{n} \\
        \leq &\, \frac{\frac{1}{n}\sum_{j | j \neq i} \exp(-\|x_j - x_i \|^2s) \| x_j - x_i \|^2s }{\frac{1}{n}\sum_{j|j \neq i} \exp(-\|x_j - x_i \|^2s)} + \log \frac{1}{n} \sum_{j|j \neq i} \exp(-\|x_j - x_i \|^2s) \\
        & + \frac{2}{n}\exp(\diam{\M}^2s).
    \end{align*}
    Thus, we obtained
    \begin{equation*}
        \overline{M}(H(P_i)-\log n) + \overline{E} \geq - \int \frac{\exp(-\| x - x_i \|^2 s)}{\int \exp(-\| \bar{x} - x_i \|^2 s) d\mu(\bar{x})} \log \frac{\exp(-\| x - x_i \|^2 s)}{\int \exp(-\| \bar{x} - x_i \|^2 s) d\mu(\bar{x})} d\mu(x)\,,
    \end{equation*}
    where
    \begin{equation*}
        \overline{M} = (1+M_1)
    \end{equation*}
    and
    \begin{equation*}
        \overline{E}=M_1 - \log(1-M_1) +\left( M_1 + \frac{2}{n}(1+M_1) \right) \exp(\diam{\M}^2s).
    \end{equation*}
    Recall that 
    \begin{equation*}
        M_1 = 2\diam{\M}s\exp(\diam{\M}^2s)\epsilon(n).
    \end{equation*}
    Since $\epsilon(n) \to 0$ as $n \to \infty$, we observe that $\lim_{n \to \infty} M_1 = 0$ for fixed $s$. Then we obtain
    \begin{align*}
        \lim_{n \to \infty } \underline{M} = \lim_{n \to \infty } ( 1 - M_1) = 1, \textrm{ and } \lim_{n \to \infty } \overline{M} = \lim_{n \to \infty } (1+M_1) = 1.
    \end{align*}
    Also,
    \begin{align*}
        &\lim_{n \to \infty } \underline{E} \\
        = &\, \lim_{n \to \infty } (-M_1 -\log(1+M_1)-M_1 \diam{\M}^2s - \frac{1-M_1}{n-1} \diam{\M}^2 s \exp(2\diam{M}^2s)) \\
        = &\, 0\,,
    \end{align*}
    and
    \begin{align*}
        & \lim_{n \to \infty } \overline{E} \\
        = & \,\lim_{n \to \infty }(M_1 - \log(1-M_1)+(M_1 \frac{2}{n}(1+M_1))\exp(\diam{\M}^2s))\\
        =&\,0\,.
    \end{align*}
    We thus finish the proof.
\end{proof}

With the help of Lemma \ref{lem: int sum entropy est}, when $n$ is sufficiently large, we can estimate $\sigma_i$ using the formula that does not depend on $n$; that is,
\begin{equation}\label{eqn: entropy int form}
   - \int \frac{\exp(-\| x - x_i \|^2/2\sigma_i^2)}{\int \exp(-\| \bar{x} - x_i \|^2/2\sigma_i^2)d\mu(\bar{x})} \log \frac{\exp(-\| x - x_i \|^2/2\sigma_i^2)}{\int \exp(-\| \bar{x} - x_i \|^2/2\sigma_i^2)d\mu(\bar{x})} d\mu(x)\,.
\end{equation}
Then, we can avoid dependency on $n$. In the next two lemmas, we show that \eqref{eqn: entropy int form} diverges to $\infty$ as $\sigma \to 0$.

\begin{Lem}
\label{lem: s norm of f est}
    For small enough $\sigma>0$, we have
    \begin{equation}
        \int \exp(-\| x - z \|^2 / 2 \sigma^2 ) d \mu(x) \sim {\sigma^m}
    \end{equation}
    for any $z \in \M$. The comparability constant is uniform over $z$.
\end{Lem}
\begin{proof}
    For simplicity, we write $s = \frac{1}{2\sigma^2}$. Let $r>0$ and divide the integral into two parts.
    \begin{align*}
        \int \exp(-\|x - z \|^2 s) d\mu(x) = & \,\int_{{B^d_r(z)}^c\cap\M} \exp(-\|x - z \|^2 s) d\mu(x)\\
        &+ \int_{B^d_r(z)\cap\M}\exp(-\|x - z \|^2 s) d\mu(x)\,.
    \end{align*}
    Noting that the function $\exp(-\rho^2 s) $ is decreasing as $\rho$ increases, the first term can be bounded from above as follows
    \begin{equation*}
        \int_{{B^d_r(z)}^c\cap\M}\exp(-\|x - z \|^2 s) d\mu(x) \leq \int_{{B^d_r(z)}^c\cap\M} \exp(-r^2 s) d\mu(x) \leq \exp(-r^2s)\,.
    \end{equation*}
    To estimate the integral in the ball $B^d_r(z)$, we will estimate with the integral on the tangent plane. Let $\exp_z : T_z \mathcal{M} \to \mathcal{M}$ be the Riemannian exponential map. Then using \eqref{eqn: riem exp bi lip}, we observe that for $r < \frac{R_\mu}{L}$,
    \begin{align*}
        &\int_{B^m_{L r}(0)} \exp(- \frac{1}{L^2} \| y \|^2 s) d {\exp_z^{-1}}_\sharp \mu(y) \\
        \geq &\,\int_{B^m_{L r}(0)} \exp(-\| \exp_z(y)-\exp_z(0) \|^2 s ) d {\exp_z^{-1}}_\sharp d\mu(y) \\
         =&\, \int_{\exp_z(B^m_{L r}(0))} \exp(-\| x -z \|^2 s ) d\mu(x) \\
         \geq&\, \int_{B^d_r(z)\cap\M} \exp(-\| x -z \|^2 s ) d\mu(x)\,,
    \end{align*}
    where we have used that $\exp_z (0) = z$. 
    Also, the bi-Lipschitzness of $\exp_z$ implies that the measure ${\exp_z^{-1}}_\sharp \mu$ is also bounded away from 0 and $\infty$ and
    \begin{equation}
        C_{L,\mu}^{-1} d \mathcal{H}^m\lfloor_{T_z\mathcal{M}} \leq d {\exp_z^{-1}}_\sharp \mu \leq C_{L,\mu} d \mathcal{H}^m\lfloor_{T_z\mathcal{M}}
        \label{equivalence of exp^-1mu and H}
    \end{equation}
    for some constant $C_{L,\mu}>0$ that only depends on $L$ and $\mu$. Then we compute
    \begin{align*}
        \int_{B^m_{L r}(0)} \exp(- \frac{1}{L^2} \| y \|^2 s ) d {\exp_z^{-1}}_\sharp \mu(y) & \leq C_{L, \mu} \int_{B^m_{L r}(0)} \exp(-\frac{1}{L^2} \| y \|^2 s) d \mathcal{H}^m \lfloor_{T_z \mathcal{M}}(y) \\
        & = \frac{C_{L,\mu}}{L^m} \int_{B_r^m(0)} \exp(-\| \tilde{y} \|^2 s) d \tilde{y},
    \end{align*}
    where $B^m_r(0)$ is a ball in $\R^m$ with radius $r$. We can compute the last integral explicitly using polar coordinate and integration by parts. Let $\beta_{m-1}$ be the volume of $(m-1)$-sphere measured with $(m-1)$-dimensional Hausdorff measure; that is, $\beta_{m-1}=\mathcal{H}^{m-1}(\mathbb{S}^{m-1})$. We have
    \begin{equation*}
        \int_{B_r^m (0)} \exp(- \| \tilde{y} \|^2 s ) d \tilde{y} = \beta_{m-1} \int_0^r \exp(-\rho^2 s ) \rho^{m-1}d\rho.
    \end{equation*}
    If $m=2$, we compute
    \begin{align*}
        \int_0^r \exp(-\rho^2 s) \rho d\rho & = \frac{1}{2}\int_0^{r^2} \exp(-\tilde{\rho} s) d\tilde{\rho} \\
        & = \frac{1}{2s}\left(1-\exp(-r^2 s) \right).
    \end{align*}
    If $m=1$, we estimate the integral as follows
    \begin{align*}\label{eqn: int exp est}
        \left( \int_0^r \exp(-\rho^2 s) d \rho \right)^2 & = \int_0^r \int_0^r \exp(-(\rho_1^2+\rho_2^2)s ) d\rho_1 d\rho_2 \\
        & = \int_{[0,r]^2} \exp(- \tilde{\rho}^2 s) \tilde{\rho} d\tilde{\rho} d \phi \\
        & \leq \frac{\pi}{2} \int_0^{\sqrt2 r} \exp(-\tilde{\rho}^2s) \tilde{\rho} d\tilde{\rho}  \numberthis\\
        & = \frac{\pi}{4s} \left(1-\exp(-2r^2s)\right).
    \end{align*}
    If $m=2k$, we use integration by parts to obtain
    \begin{align*}
        &\int_0^r \exp(-\rho^2 s) \rho \cdot \rho^{m-2} d\rho \\
         =&\, \frac{m-2}{2s} \int_0^r \exp(-\rho^2 s) \rho \cdot \rho^{m-4} d\rho - \frac{1}{2s}\exp(-r^2t) r^{m-2} \\
        & \cdots \\
         =&\, \frac{\prod_{i=1}^{k-1}(m-2i)}{(2s)^{k-1}}\int_0^r \exp(-\rho^2 s) \rho d\rho -\sum_{j=1}^{k-1} \frac{\prod_{i=1}^{j-1}(m-2i)}{(2s)^j}r^{m-2j}\exp(-r^2s) \\
        = & \,\frac{\prod_{i=1}^{k-1}(m-2i)}{(2s)^k} - \sum_{j=1}^k \frac{\prod_{i=1}^{j-1}(m-2i)}{(2s)^j}r^{m-2j} \exp(-r^2 s)\,,
    \end{align*}
    where we use the convention  $\prod_{i=1}^{0} a_i = 1$. If $m=2k+1$, the same calculation using integration by parts yields
    \begin{align*}
        &\int_0^r \exp(-\rho^2 s) \rho^m d\rho \\
        =& \, \frac{\prod_{i=1}^k(m-2i)}{(2s)^k} \int_0^r \exp(-\rho^2 s) d\rho - \sum_{j=1}^k \frac{\prod_{i=1}^{j-1}(m-2i)}{(2s)^j}r^{m-2j}\exp(-r^2s) \\
        \leq& \, \frac{\prod_{i=1}^k (m-2i)}{(2s)^k} \left( \frac{\pi}{4s} (1-\exp(-2r^2s)) \right)^{\frac{1}{2}} - \sum_{j=1}^k \frac{\prod_{i=1}^{j-1}(m-2i)}{(2s)^j}r^{m-2j}\exp(-r^2s).
    \end{align*}
    By taking the negative terms away from  both cases, we see that
    \begin{equation*}
        \int_0^r \exp(-\rho^2 s ) \rho^{m-1}d\rho \leq C_{m} \frac{1}{s^{\frac{m}{2}}}
    \end{equation*}
    for some constant $C_m$ that only depends on $m$. Therefore, we have
    \begin{equation*}
        \int \exp(-\| x-z \|^2 s) d\mu(x) \leq \exp(-r^2 s) + \frac{\beta_{m-1} C_{L,\mu} C_m}{L^{m} s^{\frac{m}{2}}}.
    \end{equation*}
    We choose $r^2 = \frac{m\log(s)}{2s}$ with a sufficiently large $s$ so that $r<\frac{R_\mu}{L}$. Then we obtain
    \begin{equation*}
        \int \exp(-\| x-z \|^2 s) d\mu(x) \leq \left(1+\frac{\beta_{m-1} C_{L,\mu}C_m}{L^m} \right) \frac{1}{s^{\frac{m}{2}}}.
    \end{equation*}
    We do a similar computation to obtain the lower bound. We first restrict the integral to a small ball to get a lower bound:
    \begin{equation*}
        \int \exp(-\| x- z\|^2 s) d\mu(x) \geq \int_{B^d_r(z)\cap\M}\exp(-\| x- z\|^2 s) d\mu(x).
    \end{equation*}
    Again, we estimate this integral with an integral on a small ball in the tangent space. We observe
    \begin{align*}
        &\int_{B^m_{\frac{r}{L}}(0)} \exp(- L^2 \| y \|^2 s) d {\exp_z^{-1}}_\sharp \mu(y)\\
         \leq&\, \int_{B^m_{\frac{r}{L}}(0)} \exp(-\| \exp_z (y) - \exp_z(0) \|^2 s) d {\exp_z^{-1}}_\sharp \mu(y) \\
         =&\, \int_{\exp(B^m_{\frac{r}{L}}(0))} \exp(-\| x - z \|^2 s) d\mu(x) \\
         \leq&\, \int_{B^d_r(z)\cap\M} \exp(- \|x - z \|^2 s ) d\mu(x)\,.
    \end{align*}
    Then we use \eqref{equivalence of exp^-1mu and H} to obtain
    \begin{align*}
        \int_{B^m_{\frac{r}{L}}(0)} \exp(-L^2 \| y \|^2s) d{\exp_z^{-1}}_\sharp \mu(y) & \geq \frac{1}{C_{L,\mu}} \int_{B^m_{\frac{r}{L}}(0)} \exp(-L^2\| y  \|^2 s) d \mathcal{H}^m\lfloor_{T_z \mathcal{M}} \\
        & = \frac{L^m}{C_{L,\mu}} \int_{B_r^m(0)} \exp(-\| \tilde{y} \|^2 s) d \tilde{y} \\
        & = \frac{L^m \beta_{m-1} }{C_{L,\mu}} \int_0^r \exp(-\rho^2 s ) \rho^{m-1}d\rho.
    \end{align*}
    We compute the last integral as in the case for the upper bound. When $m$ is even, all the computations are with equality, hence we can use the same calculation. If $m$ is odd, there is one inequality \eqref{eqn: int exp est} that we get by changing $[0,r]^2$ to $B_{\sqrt2 r}^2(0) \cap [0,\infty)^2$. We obtain an inequality in the opposite direction by using $B_r^2(0) \cap [0,\infty)^2$, which is contained in $[0,r]^2$. Then, we choose $r^2 = \frac{m \log(s)}{2s}$ with $s$ sufficiently large so that $r<L R_\mu$. Observe that
    \begin{equation*}
        \exp(-r^2s) = \frac{1}{s^{\frac{m}{2}}} \to 0
    \end{equation*}
    as $s \to \infty$ and 
    \begin{equation*}
        \frac{r^{m-2j}}{s^j} = \frac{r^m}{r^{2j} s^j} = \left(\frac{2}{m}\right)^j \frac{r^m}{(\log s)^j} = \left(\frac{2}{m}\right)^j \frac{(\log s)^{\frac{m}{2}-j}}{s^{\frac{m}{2}}} \to 0
    \end{equation*}
    as $s \to \infty$. Therefore, we see that for $m=2k$ case, we have
    \begin{align*}
        &\int_0^r \exp(-\rho^2) \rho^{m-1} d\rho \\
         = &\,\frac{\prod_{i=1}^{k-1}(m-2i)}{(2s)^k} - \sum_{j=1}^k \frac{\prod_{i=1}^{j-1}(m-2i)}{(2s)^j}r^{m-2j} \exp(-r^2 s) \\
         \geq&\, \frac{\prod_{i=1}^{k-1}(m-2i)}{(2s)^k} - \frac{1}{2} \cdot \frac{\prod_{i=1}^{k-1}(m-2i)}{(2s)^k} \\
         = &\,C'_m \frac{1}{s^\frac{m}{2}}
    \end{align*}
    when $s$ is sufficiently large. Also, in the case $m=2k+1$, we have
    \begin{align*}
        &\int_0^r \exp(-\rho^2 s) \rho^{m-1} d\rho \\
         \geq&\, \frac{\prod_{i=1}^k (m-2i)}{(2s)^k} \left( \frac{\pi}{4s} (1-\exp(-r^2s)) \right)^{\frac{1}{2}} - \sum_{j=1}^k \frac{\prod_{i=1}^{j-1}(m-2i)}{(2s)^j}r^{m-2j}\exp(-r^2s) \\
         \geq&\, \frac{\prod_{i=1}^k (m-2i)}{(2s)^k} \left( \frac{\pi}{8s} \right)^{\frac{1}{2}} - \frac{1}{2}\frac{ \prod_{i=1}^k (m-2i)}{2^k} \left( \frac{\pi}{8} \right)^{\frac{1}{2}} \frac{1}{s^{\frac{m}{2}}} \\
         = &\,C'_m \frac{1}{s^{\frac{m}{2}}},
    \end{align*}
    for sufficiently large $s$. As a result, we obtain the lower bound for the integral
    \begin{equation*}
        \int \exp(-\| x - z \|^2 s) d\mu \geq \frac{\beta_{m-1} C'_m L^m}{C_{L,\mu}} \frac{1}{s^{\frac{m}{2}}}\,.
    \end{equation*}
    Finally, we rewrite the inequalities with $\sigma$ using $ s = \frac{1}{2\sigma^2}$ and finish the proof.
\end{proof}

\begin{Lem}\label{lem: integral entropy est to infty}
\begin{equation*}
    \lim_{\sigma \to 0}  \int \frac{\exp(-\|x - z \|^2/2\sigma^2)}{\int \exp(-\| \bar{x} - z \|^2/2\sigma^2  )d\mu(\bar{x})} \log \frac{\exp(-\|x - z \|^2/2\sigma^2)}{\int \exp(-\| \bar{x} - z \|^2/2\sigma^2  )d\mu(\bar{x})} d \mu(x) = \infty.
\end{equation*}
for any $z \in \M$.
\end{Lem}
\begin{proof}
    To simplify notations, denote $f_z(x) = \exp(-\|x-z\|^2)$ and $s = \frac{1}{2\sigma^2}$. Then $s \to \infty$ as $\sigma \to 0$. Our goal is to show
    \begin{equation*}
        \lim_{s \to \infty} \int \frac{f_z(x)^s}{\| f_z \|_{L^s(\mu)}^s} \log \frac{f_z(x)^s}{\| f_z \|_{L^s(\mu)}^s} d\mu(x) = \infty\,.
    \end{equation*}
    We first divide the integral into two parts: inside and outside of a small ball.
    \begin{align*}
        & \int \frac{f_z(x)^s}{\| f \|_{L^s(\mu)}^s} \log \frac{f_z(x)^s}{\| f \|_{L^s(\mu)}^s} d\mu(x) \\
        = &\int_{B^d_r(z)\cap\M} \frac{f_z(x)^s}{\| f \|_{L^s(\mu)}^s} \log \frac{f_z(x)^s}{\| f \|_{L^s(\mu)}^s} d\mu(x) + \int_{{B^d_r(z)}^c\cap\M} \frac{f_z(x)^s}{\| f \|_{L^s(\mu)}^s} \log \frac{f_z(x)^s}{\| f \|_{L^s(\mu)}^s} d\mu(x).
    \end{align*}
    Noting that the function $t \mapsto t\log t$, $t >0$, is bounded below by $ -e^{-1}$, we can bound the integral outside the small ball from below
    \begin{equation*}
        \int_{{B^d_r(z)}^c\cap\M} \frac{f_z(x)^s}{\| f \|_{L^s(\mu)}^s} \log \frac{f_z(x)^s}{\| f \|_{L^s(\mu)}^s} d\mu(x) \geq -e^{-1}.
    \end{equation*}
    On the other hand, we use Lemma \ref{lem: s norm of f est} to see
    \begin{equation*}
        \int_{B^d_r(z)\cap\M} \frac{f_z(x)^s}{\| f \|_{L^s(\mu)}^s} \log \frac{f_z(x)^s}{\| f \|_{L^s(\mu)}^s} d\mu(x) \geq \int_{B^d_r(z)\cap\M} \frac{f_z(x)^s}{C s^{-\frac{m}{2}}} \log \frac{f_z(x)^s}{C s^{-\frac{m}{2}}} d\mu(x)
    \end{equation*}
    for some $C>0$ that depends on $m$ and $L$ (recall that $L$ is the uniform Lipschitz constant of the Riemannian exponential functions). Using the definition of $f_z$, we observe
    \begin{align*}
        \log\frac{f_z(x)^s}{C s^{-\frac{m}{2}}} & = \frac{m}{2} \log s - \| x - z \|^2 s -\log C \\
        & \geq \frac{m}{2} \log s - r^2 s - \log C
    \end{align*}
    for any $x \in B^d_r(z)\cap\M$. Therefore, choosing $r= \sqrt{\frac{m\log s}{4s}}>0$, and taking $s$ sufficiently large, we have
    \begin{equation*}
         \frac{m}{2} \log s - r^2 s - \log C \geq \frac{m}{8} \log s\,,
    \end{equation*}
    and we obtain
    \begin{equation*}
        \int_{B^d_r(z)\cap\M} \frac{f_z(x)^s}{C s^{-\frac{m}{2}}} \log \frac{f_z(x)^s}{C s^{-\frac{m}{2}}} d\mu(x) \geq \frac{ms^{\frac{m}{2}}\log s}{8C}\int_{B^d_r(z)\cap\M} f_z(x)^s d\mu(x)\,.
    \end{equation*}
    Since we have chosen $r^2 = \frac{m\log s}{4s}$, we can see from the proof of Lemma \ref{lem: s norm of f est} that 
    \begin{equation*}
        \int_{B^d_r(z)\cap\M} f_z(x)^s d\mu(x) \geq C' s^{-\frac{m}{2}}
    \end{equation*}
    for some $C'>0$ that depends on $m$ and $L$. Therefore, 
    \begin{equation*}
        \frac{ms^{\frac{m}{2}}\log s}{8C}\int_{B^d_r(z)\cap\M} f_z(x)^s d\mu(x) \geq \frac{mC' \log s}{8C},
    \end{equation*}
    and we obtain
    \begin{equation*}
        \int \frac{f_z(x)^s}{\| f_z \|_{L^s(\mu)}^s} \log \frac{f_z(x)^s}{\| f_z \|_{L^s(\mu)}^s} d\mu(x) \geq -e^{-1} + \frac{mC' \log s}{8C}.
    \end{equation*}
    Hence, the integral diverges to $\infty$ as $s \to \infty$.
\end{proof}

\begin{Rmk}
    The divergence that we proved in Lemma \ref{lem: integral entropy est to infty} is uniform over $z \in \M$. Indeed, the constants $C'$ and $C$ that show up in the last inequality in the proof of Lemma \ref{lem: integral entropy est to infty} does not depend on $z$. Therefore, for any $M>0$, we can find $s>0$ that does not depend on $z$ such that 
    \begin{equation*}
        \int \frac{f_z(x)^s}{\| f_z \|_{L^s(\mu)}^s} \log \frac{f_z(x)^s}{\| f_z \|_{L^s(\mu)}^s} d\mu(x) \geq M
    \end{equation*}
    for any $z \in \M$.
\end{Rmk}

With help of Lemma \ref{lem: int sum entropy est} and \ref{lem: integral entropy est to infty}, we can now prove that there is a uniform lower bound for $\sigma_i$ that does not depend on $n$.

\begin{Prop}\label{prop: sigma bounded away 0}
There exist constants $N_0$ and $\underline{\sigma} > 0$ that depend on $\mu$ and $\zeta$ in \eqref{our selection of perp formula}, such that if $n >N_0$, then
\begin{equation*}
    \sigma_i \geq \underline{\sigma}
\end{equation*}
for any index $i$.
\end{Prop}
\begin{proof}
    To simplify notations, we let $\frac{1}{2\sigma_i^2} = s_i$ and $\frac{1}{2\sigma^2} = s$. Also, to avoid confusion, we use $p_{j|i}(s)$ to denote the affinity \eqref{eqn: def p cond} computed with $s$; that is,
    \begin{equation*}
        p_{j|i}(s) = \frac{\exp(-\| x_i - x_j\|^2 s )}{\sum_{k \neq i} \exp(-\| x_i - x_k \|^2 s)}\,,
    \end{equation*}
    and use $P_i(s)$ to denote the probability defined with $p_{j|i}(s)$.
    By Lemma \ref{lem: integral entropy est to infty}, we can fix $\overline{s}$ such that
    \begin{equation*}
        \int \frac{\exp(-\|x - x_i \|^2\overline{s})}{\int \exp(-\| \bar{x} - x_i \|^2\overline{s}  )d\mu(\bar{x})} \log \frac{\exp(-\|x - x_i \|^2\overline{s})}{\int \exp(-\| \bar{x} - x_i \|^2\overline{s}  )d\mu(\bar{x})} d \mu(x) > - 2 \log \zeta\,.
    \end{equation*}
    Note that $\overline{s}$ depends only on $\zeta$ and the constants from Lemma \ref{lem: integral entropy est to infty} that depend on $\mu$. Then Lemma \ref{lem: int sum entropy est} implies
    \begin{align*}
        2 \log \zeta & \geq - \int \frac{\exp(-\|x - x_i \|^2\overline{s})}{\int \exp(-\| \bar{x} - x_i \|^2\overline{s}  )d\mu(\bar{x})} \log \frac{\exp(-\|x - x_i \|^2\overline{s})}{\int \exp(-\| \bar{x} - x_i \|^2\overline{s}  )d\mu(\bar{x})} d \mu(x) \\
        & \geq \underline{M} (H(P_i(\overline{s})) - \log n ) + \underline{E}\,.
    \end{align*}
    Lemma \ref{lem: int sum entropy est} also implies that we can find $N_0 > 0$ such that if $n > N_0$, then 
    \begin{equation*}
        \underline{M} = \underline{M} (n, \overline{s}) \geq \frac{1}{2} \ \ \mbox{and}\ \ \underline{E} = \underline{E}(n,\overline{s}) \geq \log \zeta\,. 
    \end{equation*}
    Hence, assuming $n \geq N_0$, we observe
    \begin{equation*}
        2\log \zeta +\log n \geq H(P_i(\overline{s}))
    \end{equation*}
    for any $i$. Taking $N_0$ larger if necessary, we can assume $\log \frac{n-1}{n} > \log \zeta$. Then the above inequality implies
    \begin{equation}\label{eqn: H(P) smaller than log perp}
        \log \zeta + \log (n-1) > H(P_i(\overline{s})).
    \end{equation}
    Noting that the Shannon entropy $H(P_i(s))$ is a decreasing function of $s$ by Lemma \ref{Lem: perp monotone}, we deduce that if $s>\overline{s}$, then we have \eqref{eqn: H(P) smaller than log perp} with $H(P_i(s))$. On the other hand, we have 
    \begin{equation*}
        \log Perp = H(P_i(s_i))\,.
    \end{equation*}
    Then from Lemma \ref{Lem: perp monotone}, we obtain that $s_i \leq \overline{s}$ for any $i$, and hence 
    \begin{equation*}
        \sigma_i \geq \frac{1}{\sqrt{2\overline{s}}}
    \end{equation*}
    for any $i$. Since $\overline{s}$ was decided by $\zeta$ and $\mu$, we obtain the claim with $\underline{\sigma} = \frac{1}{\sqrt{2\overline{s}}}$.
\end{proof}

\begin{Cor}\label{cor: pij est}
    Let $n> N_0 $ and denote $C_p = \exp(\frac{\diam{\M}}{2\underline{\sigma}^2})>1$. Then we have
    \begin{equation}\label{eqn: cond p comp 1/n}
        \frac{C_p^{-1}}{n-1} \leq p_{j|i} \leq \frac{C_p}{n-1}
    \end{equation}
    and
    \begin{equation}\label{eqn: p comp 1/n2}
        \frac{2C_p^{-1}}{n(n-1)} \leq p_{ij} \leq \frac{2C_p}{n(n-1)}
    \end{equation}
    for any $i \neq j$.
\end{Cor}
\begin{proof}
    Thanks to Proposition \ref{prop: sigma bounded away 0}, we observe that
    \begin{equation*}
        \log \frac{p_{j|i}}{p_{k|i}} \leq \frac{\diam{M}^2}{\underline{\sigma}^2}.
    \end{equation*}
    Therefore, we obtain $C_p^{-1} p_{k|i} \leq p_{j|i} \leq C_p p_{k|i}$. Then we have
    \begin{equation*}
        (n-1) C_p^{-1} p_{j|i} \leq 1 = \sum_{k|k \neq i} p_{k|i} \leq (n-1) C_p p_{j|i},
    \end{equation*}
    which proves \eqref{eqn: cond p comp 1/n}. Also, by definition of $p_{ij}$ \eqref{eqn: def p}, we obtain
    \begin{equation*}
        \frac{2C_p^{-1}}{n(n-1)} \leq p_{ij} = \frac{1}{n}(p_{j|i}+p_{i|j}) \leq \frac{2C_p}{n(n-1)}.
    \end{equation*}
\end{proof}

Next, we consider $q_{ij}'$ defined in \eqref{definition q'ij}. The points $y_i$ are initially picked up randomly hence we do not refer to a certain measure. In fact, the structure of $q_{ij}'$ gives information about $\sum_{i,j|i \neq j}{q_{ij}'}^2$.

\begin{Lem}\label{lem: sum q^2 est}
    Let $\{ y_i\} \subset \R^2$ be a set of $n$ distinct points. Then we have the following inequality
    \begin{equation*}
       \sum_{i,j|i \neq j}{q_{ij}'}^2= \frac{\sum_{i,j|i \neq j}\|y_i - y_j \|^{-4}}{\left( \sum_{i,j|i \neq j} \| y_i - y_j \|^{-2}\right)^2} \geq \frac{1}{4n(\log n)^2}\,.
    \end{equation*} 
\end{Lem}
\begin{proof}
    Define $r_i = \min_{j|j \neq i} \| y_j - y_i\|$, and re-index $y_i$ so that $r_1 \leq r_2 \leq \cdots \leq r_n$. Noting that $\frac{\sum_{i,j| i \neq j}\|y_i - y_j \|^{-4}}{\left( \sum_{i,j| i \neq j} \| y_i - y_j \|^{-2}\right)^2}$ is invariant under scaling on the points $y_i$, we assume 
    \begin{equation*}
        \sum_i r_i^{-2} = 1.
    \end{equation*}
    Let $y_i^*$ be the point in $\{ y_j \}_{j=1}^n$ such that $r_i = \| y_i^*-y_i\| $. Then we use Cauchy-Schwartz inequality to obtain
    \begin{equation}\label{eqn: d_ij^-4 cs}
        \sum_{i,j|i \neq j} \| y_i - y_j \|^{-4} \geq \sum_{i} \| y_i^*-y_i \|^{-4} =  \sum_i r_i^{-4} \geq \frac{1}{n} \left( \sum_i r_i^{-2} \right)^2 = \frac{1}{n}\,.
    \end{equation}
    Next we estimate $\sum_{i,j| i \neq j}\|y_i - y_j \|^{-2}$. We define $A_j^i = B_{\frac{r_i}{2}}(y_j) \subset \R^2$, a ball with radius $\frac{r_i}{2}$ centered at $y_j$. We claim that $A_j^i \cap A_k^i = \emptyset$ for any triple $i<j<k$. Indeed, $r_i \leq r_j \leq r_k$ implies $y_k \not\in B_{r_i}(y_j)$ and we deduce that $B_{\frac{r_i}{2}}(y_j) \cap B_{\frac{r_i}{2}}(y_k) = \emptyset$. We also claim the following: for $i < j$,
    \begin{equation}\label{eqn: dist to integral}
        \| y_i - y_j \|^{-2} \leq \frac{1}{r_i^2\pi} \int_{A_j^i} \| y_i - y \|^{-2} dy\,.
    \end{equation}
    To show the claim, we observe that $y\in A_j^i$ implies 
    \begin{equation*}
        \| y_i - y \| \leq \| y_i - y_j \| + \frac{r_i}{2} \leq 2 \| y_i - y_j \|\,,
    \end{equation*}
    where we have used the definition of $r_i$ in the second inequality. Therefore, noting that $|A_j^i| = \frac{r_i^2}{4}\pi$, we compute
    \begin{equation*}
        \| y_i - y_j \|^{-2} = \frac{1}{| A_j^i |} \int_{A_j^i} \| y_i - y_j \|^{-2} dy \leq \frac{1}{r_i^2 \pi} \int_{A_j^i} \| y_i - y \|^{-2} dy.
    \end{equation*}
    Now we estimate $\sum_{j|j>i} \| y_i - y_j \|^{-2}$. We use \eqref{eqn: dist to integral} and obtain
    \begin{equation*}
        \sum_{j|j>i} \| y_i - y_j \|^{-2} \leq \frac{1}{r_i^2\pi} \sum_{j|j>i} \int_{A_j^i} \| y_i - y \|^{-2} dy.
    \end{equation*}
    Let $\mathcal{A}_i = B_{R_i}(y_i) \setminus B_{\frac{r_i}{2}}(y_i)$ be an annulus where $R_i = \frac{r_i\sqrt{n}}{2} $. Note that $| \mathcal{A}_i | = \frac{r_i^2}{4}\pi(n-1)$ and $|A_j^i| = \frac{r_i^2}{4}\pi$. Then we claim
    \begin{equation*}
        \sum_{j|j>i} \int_{A_j^i} \| y_i - y \|^{-2} dy \leq \int_{\mathcal{A}_i} \| y_i - y \|^{-2} dy.
    \end{equation*}
    To prove the claim, set 
    \begin{equation*}
        \mathcal{A}_i^1 = \mathcal{A}_i \cap \left( \bigcup_{j|j>i} A_j^i \right),\ \ \mathcal{A}_i^2 = \mathcal{A}_i \setminus \left( \bigcup_{j|j>i} A_j^i \right)\ \ \mbox{and}\ \ \mathcal{A}_i^3 = \left( \bigcup_{j|j>i} A_j^i \right) \setminus \mathcal{A}_i\,.
    \end{equation*} 
    Then 
    \begin{equation*}
        \mathcal{A}_i = \mathcal{A}_i^1 \cup \mathcal{A}_i^2\ \ \mbox{and}\ \ \left( \bigcup_{j|j>i} A_j^i \right) = \mathcal{A}_i^1 \cup \mathcal{A}_i^3\,.
    \end{equation*} 
    Also,
    \begin{equation*}
        \sup \{ \| y-y_i \|  | y \in \mathcal{A}_i^2 \} \leq R_i \leq \mathrm{dist}(y_i, \mathcal{A}_i^3)\ \ \textrm{ and }\ \ |\mathcal{A}_i^2 | = | \mathcal{A}_i^3 |.
    \end{equation*}
    Therefore,
    \begin{align*}
        \int_{\mathcal{A}_i} \| y_i - y \|^{-2} dy & = \int_{\mathcal{A}_i^1} \| y_i - y \|^{-2} dy + \int_{\mathcal{A}_i^2} \| y_i - y \|^{-2} dy \\
        & \geq \int_{\mathcal{A}_i^1} \| y_i - y \|^{-2} dy + | \mathcal{A}_i^2 | R_i^{-2} \\
        & = \int_{\mathcal{A}_i^1} \| y_i - y \|^{-2} dy + | \mathcal{A}_i^3 | R_i^{-2} \\
        & \geq \int_{\mathcal{A}_i^1} \| y_i - y \|^{-2} dy + \int_{\mathcal{A}_i^3} \| y_i - y \|^{-2} dy \\
        & = \sum_{j|j>i} \int_{A_j^i} \| y_i - y \|^{-2} dy\,.
    \end{align*}
    We compute $\int_{\mathcal{A}_i} \| y_i - y \|^{-2} dy$ using the polar coordinate centered at $y_i$:
    \begin{align*}
        \int_{\mathcal{A}_i} \| y_i - y \|^{-2} dy & = \int_0^{2\pi} \int_{\frac{r_i}{2}}^{R_i} \rho^{-2} \cdot \rho d\rho d\theta \\
        & = 2\pi \int_{\frac{r_i}{2}}^{\frac{r_i\sqrt{n}}{2}} \rho^{-1} d\rho \numberthis \label{eqn: int dist^-2}\\
        & = \pi \log n 
    \end{align*}
    Therefore, we have
    \begin{equation}\label{eqn: d_ij^-2 bound}
        \sum_{i,j| i \neq j} \| y_i - y_j \|^{-2} = 2 \sum_i \sum_{j|j>i} \| y_i - y_j \|^{-2} \leq 2\sum_{i} \frac{\pi \log n}{r_i^2 \pi}  = 2 \log n\,,
    \end{equation}
    where we use the assumption $\sum_i \frac{1}{r_i^2} = 1$ in the last equality. In summary, we use \eqref{eqn: d_ij^-4 cs} and \eqref{eqn: d_ij^-2 bound} to obtain
    \begin{equation*}
        \frac{\sum_{i,j| i \neq j} \|y_i - y_j \|^{-4}}{\left( \sum_{i,j| i \neq j} \|y_i - y_j \|^{-2} \right)^2} \geq \frac{1}{4n( \log n)^2}.
    \end{equation*}
\end{proof}

\begin{Rmk}
    The order of estimate that we obtain in the Lemma \ref{lem: sum q^2 est} depends on the dimension of $\R^2$, where the points $y_i$ are in. Indeed, we used this dimension condition in \eqref{eqn: int dist^-2} to obtain the $\log n$ factor. If we change the dimension of the target space of t-SNE, the estimate we obtain from Lemma \ref{lem: sum q^2 est} changes. For example, if we use $\R^k$ with $k\geq 3$, then a similar proof yields
    \begin{equation*}
        \sum_{i,j|i \neq j} {q_{ij}'}^2 \gtrsim n^{-(3-\frac{4}{k})}.
    \end{equation*}
    In particular, if $k=3$, then we obtain $\sum_{i,j|i \neq j} {q_{ij}'}^2 \gtrsim n^{-\frac{5}{3}}$ and we obtain the argument that we discussed below equation \eqref{eqn: diff sum up bound'}. If $k \geq 4$, however, the argument cannot be applied since the order of $\sum_{i,j|i \neq j}{q_{ij}'}^2$ will be small than or equal to $n^{-2}$ so that we cannot obtain $\sum_{i,j|i \neq j}p_{ij}^2 < \sum_{i,j|i \neq j}{q_{ij}'}^2$. One possible way to make the argument work is to change the formula of $q_{ij}$. This is however out of the scope of this paper.
\end{Rmk}

\subsection{The first main theorem: Boundedness of $\{ y_i \}$}

Armed with the above discussion and results, we are now ready to state our first main theorem. 
\begin{Thm}\label{thm: bounded}
    Take $N_0$ in Proposition \ref{prop: sigma bounded away 0} and $C_p$ in Corollary \ref{cor: pij est}. Fix $n>N_0$ such that $\frac{n-1}{(\log n)^2} \geq 8 C_p^2$. Then there exists $R_n>0$ such that
    \begin{equation*}
        \{ y_i(t) \}_{i=1}^n \subset B^2_{R_n}(0)
    \end{equation*}
    for any $t>0$.
\end{Thm}
\begin{proof}
    Suppose, in contrast, that there is no such $R_n$. Then we obtain an index $1 \leq a \leq n$ and $0 < t_\infty \leq \infty$ such that $y_a$ is defined on $[0,t_\infty)$ and 
    \begin{equation*}
        \lim_{t \nearrow t_{\infty}} y_a(t) = \infty.
    \end{equation*}
    Then by Lemma \ref{lem: diverging pair} and Lemma \ref{lem: unif dist ratio}, we have
    \begin{equation*}
        \lim_{t \nearrow t_{\infty}} \|y_i - y_j\| = \infty
    \end{equation*}
    for any $i,j$. In particular, there exists $t_1 \in (0,t_\infty)$ such that if $t>t_1$,
    \begin{equation*}
        \| y_i - y_j \| > 1 \ \ \mbox{for all}\ \ i\neq j.
    \end{equation*}
    Then, by \eqref{eqn: diff sum up bound'}, \eqref{eqn: p comp 1/n2}, Lemma \ref{lem: sum q^2 est}, and our choice of $n$, we obtain
    \begin{equation}\label{eqn: diff dij <0}
        \frac{d}{dt} \sum_{i,j|i \neq j} \| y_i(t) - y_j(t) \|^2 <0
    \end{equation}
    for any $t>t_1$. This implies
    \begin{equation*}
        \infty = \lim_{t \nearrow t_\infty} \sum_{i,j|i \neq j} \| y_i(t) - y_j(t) \|^2 < \sum_{i,j| i \neq j} \| y_i(t_1) - y_j(t_1) \|^2 < \infty\,,
    \end{equation*}
    which is a contradiction. Therefore $\{ y_i (t) \}_{i=1}^n$ must be bounded.
\end{proof}

Inequality \eqref{eqn: diff dij <0} holds whenever all the mutual distance $\| y_i - y_j \|$ is greater than 1. Therefore, if $\frac{d}{dt} \sum_{i,j|i \neq j} \|y_i - y_j \|^2 \geq0$, then there must be at least one pair $i\neq j$ such that $\| y_i - y_j \|<1$. With this in mind, we can compute explicit bound for $y_i$ using the idea in Remark \ref{rmk: explicit ratio}.

\begin{Cor}\label{cor: explicit bound}
    Under the same assumptions in Theorem \ref{thm: bounded}, there exists $\underline{\tau}>0$ such that if $t>\underline{\tau}$ then we have
    \begin{equation*}
        \| y_i (t) \| \leq \sqrt{2} \exp\left(\frac{C_p}{2}n(n-1)\Big(\cost_0- \sum_{k,l|k \neq l} p_{kl} \log p_{kl}\Big)\right),
    \end{equation*}
    where $\cost_0 = \cost(\Y^{(0)})$ is the initial value of the KL-divergence.
\end{Cor}
\begin{proof}
    Let $\underline{\tau} = \min\{t>0 | \frac{d}{dt}\sum_{i,j|i \neq j} \|y_i - y_j \|^2 \geq 0\}$, and let $\tau > \underline{\tau}$. We divide the proof into two cases depending on the sign of $\frac{d}{dt}\sum_{i,j|i \neq j} \|y_i - y_j \|^2 \big|_{t=\tau} $.
    
    Suppose first that $\frac{d}{dt} \sum_{i,j| i \neq j} \| y_i - y_j \|^2 \big|_{t=\tau} \geq 0 $. Then by the above argument with \eqref{eqn: diff dij <0}, we can assume that there is a pair $a\neq b$ such that $\| y_a(\tau) - y_b(\tau) \| \leq 1$. Then, for any $i\neq j$, we have
    \begin{align*}
        q_{ij}(\tau) & = \frac{(1+\| y_i(\tau) -y_j(\tau) \|^2)^{-1}}{\sum_{k,l|k \neq l}(1+\| y_k(\tau) -y_l(\tau) \|^2)^{-1}} \\
        & \leq \frac{\| y_i(\tau) - y_j(\tau) \|^{-2}}{(1+\| y_a(\tau) - y_b(\tau) \|^2)^{-1}} \\
        & \leq \frac{2}{\| y_i(\tau) - y_j(\tau) \|^2}.
    \end{align*}
    On the other hand, since the KL-divergence is a decreasing function of $t$, we observe
    \begin{align*}
        \cost_0 & \geq \cost(\Y^{(\tau)}) = \sum_{k,l| k \neq l} p_{kl}\log \frac{p_{kl}}{q_{kl}(\tau)} \\
        & \geq \sum_{k,l|k \neq l} p_{kl} \log p_{kl} + p_{ij} \log \frac{1}{q_{ij}(\tau)}.
    \end{align*}
    Therefore, by Corollary \ref{cor: pij est} we observe
    \begin{align*}
        \| y_i(\tau) - y_j(\tau) \|^2 & \leq \frac{2}{q_{ij}(\tau)} \leq 2 \exp\left( (\cost_0 - \sum_{k,l| k \neq l} p_{kl}\log p_{kl})/p_{ij}\right) \\
        & \leq 2\exp\left( C_p n(n-1) (\cost_0 - \sum_{k,l| k \neq l} p_{kl}\log p_{kl})\right)
    \end{align*}
    and hence
    \begin{equation*}
        \| y_i(\tau) - y_j(\tau) \| \leq \sqrt2 \exp\left( \frac{C_p}{2} n(n-1) (\cost_0 - \sum_{k,l| k \neq l} p_{kl}\log p_{kl})\right).
    \end{equation*}
    Finally, we use the assumption $\sum_{i=1}^n y_i = 0$ and obtain
    \begin{align*}
        \| y_i(\tau) \| & = \Big\| \frac{1}{n} \sum_{j=1}^n (y_j(\tau) - y_i(\tau))\Big\| \\ 
        & \leq \frac{1}{n} \sum_{j=1}^n \| y_j(\tau) - y_i(\tau) \| \leq \sqrt2 \exp\left( \frac{C_p}{2} n(n-1) (\cost_0 - \sum_{k,l| k \neq l} p_{kl}\log p_{kl})\right).
    \end{align*}
    This proves the corollary in the case $\frac{d}{dt} \sum_{i,j| i \neq j} \| y_i - y_j \|^2 \big|_{t=\tau} \geq 0 $.
    
    Next, suppose we have $\frac{d}{dt} \sum_{i,j|i \neq j} \| y_i - y_j \|^2 \big|_{t=\tau} <0$. Define
    \begin{equation*}
        t_+ := \sup \left\{ s < \tau \,\Big|\, \frac{d}{dt} \| y_i - y_j \|^2 \big|_{t=s} \geq 0 \right\}.
    \end{equation*}
    Note that since it is assumed that $\tau >\underline{\tau}$ and by the definition of $\underline{\tau}$, $t_+$ is well-defined. Then we have that for any $t_+<s<\tau$, we have $\frac{d}{dt} \sum_{i,j|i \neq j} \|y_i - y_j \|^2 \big|_{t=s}<0$, i.e. $\sum_{i,j|i \neq j} \|y_i - y_j \|^2$ is decreasing on the interval $(t_+,\tau)$. We also have that 
    \begin{equation*}
        \frac{d}{dt} \sum_{i,j|i \neq j} \|y_i - y_j \|^2 \big|_{t=t_+}=0.
    \end{equation*}
    Therefore, we obtain
    \begin{align*}
        \sum_{i,j| i \neq j} \| y_i(\tau) - y_j(\tau) \|^2 & \leq \sum_{i,j| i \neq j} \| y_i(t_+) - y_j(t_+) \|^2 \\
        & \leq 2n \exp\left(C_p n(n-1)(\cost_0-\sum_{i,j|i \neq j} p_{ij}\log p_{ij}) \right),
    \end{align*}
    where we have used the first part of the proof in the second inequality. Then, we use $\sum_{i=1}^n y_i = 0$ again to obtain
    \begin{align*}
        \| y_i (\tau) \|^2 & = \Big\| \frac{1}{n} \sum_{j=1}^n (y_i(\tau) - y_j(\tau)) \Big\|^2  \leq \frac{1}{n} \sum_{j=1}^n \| y_i(\tau) - y_j(\tau) \|^2 \\
        & \leq 2 \exp\left({C_p} n(n-1)(\cost_0-\sum_{i,j|i \neq j} p_{ij}\log p_{ij}) \right),
    \end{align*}
    where we have used Jensen's inequality in the second line. Taking square root, we obtain the desired inequality in the case $\frac{d}{dt} \sum_{i,j|i \neq j} \|y_i - y_j \|^2 \big|_{t=\tau}<0$.
\end{proof}

\subsection{The second main theorem:  existence of a minimizer}

The bound in Corollary \ref{cor: explicit bound} depends on $n$, the initial value of KL-divergence $\cost_0$, and the affinity of original data $p_{ij}$. Therefore, once we fix the data points $\{ x_i \}_{i=1}^n$, the gradient flow \eqref{eqn: grad desc conti} with a set of arbitrary $n$ points produces uniformly bounded curves as long as the initial values of the KL-divergences are bounded. Then, a minimizing sequence $\mathcal{Y}_k = \{ y_{k,i} \}_{i=1}^n$ can be bounded uniformly, and we can use the compactness to show that there exists a minimizer.

\begin{Thm}
    Take $N_0$ in Proposition \ref{prop: sigma bounded away 0} and $C_p$ in Corollary \ref{cor: pij est}. Fix $n>N_0$ such that $\frac{n-1}{(\log n)^2} \geq 8 C_p^2$. Then there exists a global minimizer of the KL-divergence.
\end{Thm}
\begin{proof}
    Let $\mathcal{Y}_k = \{ y_{k,i} \}_{i =1}^n$ be a sequence of sets of points in $\R^2$ that minimizes the KL-divergence; that is,
    \begin{equation*}
        \inf_{\Y=\{y_i \} \subset \R^2} \cost(\Y) = \lim_{k \to \infty} \cost_k\,,
    \end{equation*}
    where $\cost_k = \cost(\Y_k)$ is the  KL-divergence computed with $\mathcal{Y}_k$.
    We can assume that the sequence $\cost_k$ is decreasing. By Corollary \ref{cor: explicit bound}, for each $k$, there exists $t_k$ such that if $t>t_k$, then 
    \begin{equation*}
        \mathcal{Y}^k = \{ y_{k,i} (t) \}_{i=1}^n \subset B^2_{R_{p,k}}(0),
    \end{equation*}
    where $y_{k,i}$ satisfies equation \eqref{eqn: grad desc conti} with initial condition $y_{k,i} (0) = y_{k,i}$ and $R_{p,k}$ is given by
    \begin{equation*}
        R_{p,k} = \sqrt{2}\exp\left( \frac{C_p}{2}n(n-1) (\cost_k-\sum_{i,j| i \neq j } p_{ij} \log p_{ij})\right).
    \end{equation*}
    Note that $R_{p,k}$ is a decreasing sequence in $k$ since $\cost_k$ is decreasing. Therefore, letting $\mathcal{Y}_k' = \{y_{k,i}(t_k) \}$, we have $\mathcal{Y}_k' \subset B^2_{R_{p,1}}(0)$ for any $k$. Let $\cost_k' = \cost(\Y_k')$ be the KL-divergence computed with $\mathcal{Y}_k'$. Then we also have $\cost_k' \leq \cost_k$. Therefore, $\cost_k'$ is another minimizing sequence of the KL-divergence, but it is also uniformly bounded. Therefore, up to a subsequence, we can assume that $\lim_{k \to \infty} y_{k,i}(t_k) = z_i$ for some $z_i \in \overline{B^2_{R_{p,1}}(0)}$ for any $i$. Then, since the KL-divergence depends on the points $y_i$ continuously, we obtain that
    \begin{equation*}
        \inf_{\Y=\{y_i \} \subset \R^2} \cost(\Y) = \lim_{k \to \infty} \cost_k \geq \lim_{k \to \infty} \cost_k' = \cost_\infty',
    \end{equation*}
    where $\cost_\infty' = \cost(\mathcal{Z})$ is the KL-divergence computed with the set of points $\mathcal{Z}=\{ z_i \}_{i=1}^n$. Hence, $\{ z_i \}_{i=1}^n$ is a minimizer of the KL-divergence.
\end{proof}

After establishing the existence of a minimizer, questions regarding its uniqueness naturally arise. However, due to the structure of $q_{ij}$, multiple minimizers can be easily identified. The affinity $q_{ij}$ is decided by the mutual distances, and therefore, applying an isometric transform on the points $\{ y_i \}_{i=1}^n$ does not change the KL-divergence. 
Consequently, isometric transforms of a minimizer provide multiple distinct minimizers. Therefore, a more appropriate question is whether the minimizer is unique up to isometric transforms. This aspect will be explored in our future work.

\section*{Acknowledgement}
This research was partially discussed with Professor Chih-Wei Chen from National Sun Yat-Sen University, Taiwan. The authors express their gratitude to him for his valuable insights and discussion.

\bibliographystyle{amsalpha}
\bibliography{Theoretical_t-SNE.bib}
\end{document}